\RequirePackage{amsmath}
\RequirePackage{fix-cm}

\documentclass[twocolumn]{svjour3}

\smartqed  

\usepackage{graphicx}

\usepackage{csquotes}
\usepackage{multirow}
\usepackage{mathtools}
\usepackage{amssymb, amsfonts}
\usepackage{chngpage}
\usepackage{marginnote}
\usepackage{paralist}
\usepackage[normalem]{ulem}
\usepackage{sidecap}
\usepackage{float}
\usepackage{caption3}
\DeclareCaptionOption{parskip}[]{}

\usepackage{rotating}
\usepackage{subfigure}

\usepackage{complexity}
\usepackage[ruled, linesnumbered, boxed, noend]{algorithm2e}
\DontPrintSemicolon
\usepackage{placeins}
\usepackage{booktabs}
\usepackage{widetext} 
\usepackage{breqn}
\usepackage{multirow}
\usepackage{microtype}

\usepackage[hyphens]{url}
\usepackage[breaklinks]{hyperref}
\usepackage{bm}
\DeclareMathOperator*{\argmin}{arg\,min}
\DeclareMathOperator*{\argmax}{arg\,max}
\DeclareMathOperator{\vect}{vec}

\usepackage{natbib}

\newcommand{\sourcecodeurl}{\href{http://github.com/bermanmaxim/probablySubmodular}{\texttt{github.com/\allowbreak bermanmaxim/\allowbreak probablySubmodular}}}

\begin{document}

\title{Discriminative training of conditional random fields with probably submodular constraints}

\author{Maxim Berman \and Matthew B. Blaschko}

\institute{Maxim Berman \at
              Dept.\ of Electrical Engineering, KU Leuven, 3001 Leuven, Belgium \\
              \email{maxim.berman@esat.kuleuven.be}           
           \and
Matthew B. Blaschko \at
              Dept.\ of Electrical Engineering, KU Leuven, 3001 Leuven, Belgium \\
              \email{matthew.blaschko@esat.kuleuven.be}
}

\date{}

\maketitle

\begin{abstract}
Problems of segmentation, denoising, registration and 3D reconstruction are often addressed with the graph cut algorithm. However, solving an unconstrained graph cut problem is NP-hard. For tractable optimization, pairwise potentials have to fulfill the submodularity inequality. In our learning paradigm, pairwise potentials are created as the dot product of a learned vector $w$ with positive feature vectors. In order to constrain such a model to remain tractable, previous approaches have enforced the weight vector to be positive for pairwise potentials in which the labels differ, and set pairwise potentials to zero in the case that the label remains the same. Such constraints are sufficient to guarantee that the resulting pairwise potentials satisfy the submodularity inequality. However, we show that such an approach unnecessarily restricts the capacity of the learned models. 
Guaranteeing submodularity for all possible inputs, no matter how improbable, reduces inference error to effectively zero, but increases model error.  In contrast, we relax the requirement of guaranteed submodularity to solutions that are probably approximately submodular.  We show that the conceptually simple strategy of enforcing submodularity on the training examples guarantees with low sample complexity that test images will also yield submodular pairwise potentials.  Results are presented in the binary and muticlass settings, showing substantial improvement from the resulting increased model capacity.

\end{abstract}

\def\w{\mathbf{w}}     
\def\fpsi{\bm{\psi}}   
\def\feat{\bm{\phi}}   
\def\E{\mathcal{E}}    
\def\V{\mathcal{V}}    
\def\ind{\mathbf{1}}   
\def\lab{\mathcal{L}}  
\def\T{\mathcal{S}}    
\def\C{\mathcal{C}}    
\def\fp{d}             
\newcommand{\matr}[1]{\mathbf{#1}}
\def\Co{\C_0}             
\def\Co{\C_2}             
\def\Cf{\C_4}             

\def\Ctrans{$\overline{\Cf}$}
\def\NP{$\mathsf{NP}$}

\section{Introduction}

Multiple problems emerging in computer vision, such as segmentation, denoising, registration and 3D reconstruction, are addressed with Structured Output Support Vector Machines (SSVM) applied to conditional random field (CRF) models. The arising problem of energy minimization in CRFs can be solved by a variety of methods, including loopy belief propagation, alpha-expansion, alpha-beta swap and many others. A majority of energy minimization algorithms require pairwise potentials to fulfill (pairwise) submodular constraints or metric constraints. This requirement places a strong limitation on the family of models that can be employed. 

It is well known from statistical learning theory that the prediction error of a discriminant function can be decomposed into the error resulting from the learning procedure, and the error resulting from the model class~\citep{BartlettetalJASA2006}.  In a structured output setting, such as in learning the parameters of a CRF model, a method may also have error resulting from suboptimal inference.  In this work we explore the tradeoffs resulting from this third source of error, showing that \begin{inparaenum}[(\itshape a\upshape)]
\item increasing model capacity by allowing some test-time potentials to be potentially non-submodular generally improves accuracies over guaranteeing submodularity for all possible inputs and
\item we can bound the probability of a non-submodular constraint occurring at test time with low sample complexity.
\end{inparaenum}
  This latter result indicates that relaxing submodularity constraints to guarantee ``only'' probably submodular potentials is a safe and principled strategy for increasing model capacity and increasing the resulting system accuracy.

\renewcommand{\UrlBreaks}{\do\/}

In this work, we make several fundamental contributions to discriminative learning of CRF models: 
\begin{inparaenum}[(\itshape a\upshape)]
 \item a formulation for learning models with probably submodular cons\-traints,
 \item an algorithm for efficiently generating the most violated submodularity constraint,
 \item the concept of a tradeoff between model error and inference error in CRF training, and
 \item empirical results showing substantial improvement on segmentation and multi-label classification datasets. 
\end{inparaenum}
Source code of the learning algorithms presented here are available for download from \sourcecodeurl .

This paper is an extended version of~\citep{Zaremba2016a}.  \citet{Zaremba2016a} have built on existing frameworks for discriminative CRF training for image segmentation by:
\begin{inparaenum}[(i)]
\item introducing necessary and sufficient conditions to ensure submodularity of the inference process (Equation~\eqref{eq:restrsubmodular}) and the notion of probably submodular constraints (Section~\ref{sec:CRFprobablySubmodular}),
\item developing binary label semantic segmentation applications (Section~\ref{sec:toyseg}), and
\item an optimization scheme for the binary setting (Equation~\eqref{eq:ConstraintMarginsSpeedupTensorFactor} and Algorithm~\ref{alg:QP_naive}). 
\end{inparaenum}
We extend our previous work by adding the following contributions:
\begin{inparaenum}[(i)]
\item an improved presentation of the out of sample behaviour of the probably submodular setting (Section~\ref{sec:SampleComplexity}),
\item an improved training scheme with increased computational efficiency (Section~\ref{sec:DelayedConstraintGenerationNIPSWorkshop}),
\item additional segmentation applications (Sections~\ref{sec:ApplicationPhotonReceptorCells2Dretinal} and~\ref{sec:ApplicationMitochondria}),
\item extension of the framework to the multi-label setting (Sections~\ref{sec:MultiLabelClassificationApplicationSetting} and~\ref{sec:ExperimentsMultiLabelClassification}).
\end{inparaenum}

\subsection{Related work}

Random field models in image segmentation initially employed data independent pairwise terms encoding a relatively simple prior that adjacent pixels were likely to have the same label~\citep{GemanG84,BoykovVZ01}.  The first data dependent pairwise terms proposed in the literature were simple contrast dependent terms with fixed positive weighting, resulting in a guarantee of submodularity~\citep{BoykovJ01}.  In the first applications of structured output support vector machines to the discriminative learning of pairwise terms, only associative potentials were employed, enforced by a single positive constraint~\citep{Anguelov2005}. A later work employed only two positively constrained learned weights: one for a Potts-like term, and one for a contrast dependent term~\citep{SzummerKH08}. 
This simple positivity constraint is sufficient to guarantee submodularity for all possible inputs, but does not give the learning algorithm much capacity to optimize the pairwise terms.  In contrast, we consider here the optimization of hundreds or thousands of pairwise parameters, providing a rich model space for learning informative pairwise potentials.  An alternative approach is to consider only tree structured models~\citep{NowozinGL10}, but this again restricts the model space and disallows potentially helpful model interactions.

In relaxing the constraint set to include models that do not guarantee submodularity for all possible inputs, we develop bounds on the probability of a test time input resulting in a non-submodular potential.  This problem reduces to the problem of estimating the sample complexity of learning a convex cone by an intersection of half spaces.

The approach of bounding the error of an algorithm is closely related to the notion of probably approximately correct (PAC) learning~\citep{Valiant84}.
In analogy to PAC learning, probabilistic bounds have been considered before in the development of inference algorithms for computer vision problems, such as in the development of thresholds for an object detection cascade architecture~\citep{Felzenszwalbcascade}.

Contemporary semantic segmentation methods almost always resort to deep learning, as e.g. the U-Net architecture popular in medical computer vision applications~\citep{UNetMICCAI2015}. 
However, there have been successful integrations of conditional random field structure within deep vision applications~\citep{ChenSchwingICML2015,ChandraEccv2016,chen2017deeplab}, which pave the way to the transfer of our probabilistic improvement to the expressivity of exact CRF models to large-scale vision applications. 
Moreover, learning on hard-coded features or with exact inference remains of relevance in systems with limited data, or where strong robustness guaranties 
are needed, as deep learning-based solutions can be prone to adversarial attacks~\citep{Metzen_2017_ICCV}. 
In this work, we focus on a setting with a fixed feature representation, which allows us to study in a principled fashion the coupling between the features and the model weights. 
This also allows us to frame to problem of probably submodular learning as a framework more general than semantic segmentation, which, in particular, we also apply to multi-label classification problems.

\section{Discriminative training of segmentation models}\label{sec:model}

We consider the learning framework of Structured Support Vector Machines (SSVM) of~\citet{TsochantaridisJHA05}, a large-margin classifier suited to problems with structured input-output spaces, such as image segmentation and multi-label problems \citep{Bakir:2007:PSD:1296180}.

The SSVM learns a function from a general input space $\mathcal{X}$ to an output spaces $\mathcal{Y}$ by mapping an element $x \in \mathcal{X}$ to a solution $y^*$ of the \emph{MAP inference problem} 
\begin{equation}\label{map_problem}
y^* = \argmax_{y\in\mathcal{Y}}{\w^\intercal \fpsi(x, y)}
\end{equation}
where $\vec{w} \in \mathbb{R}^d$ is the weight vector of the model and $\fpsi: \mathcal{X} \times \mathcal{Y} \rightarrow \mathbb{R}^d$ is the joint feature map encoding the joint structure of the input-output space. 

Following the 1-slack margin-rescaling formulation of the SSVM \citep{joachimsml09}, 
the weight vector $\w$ can be discriminatively learned from a training set $\T$ of $n$ training samples $(x_i, y_i)_{i=1\ldots n}$, by minimizing the regularized large-margin objective
\begin{equation}
\min_{\w, \xi}\enskip \frac{1}{2} \lVert \w \lVert ^2 + C \xi \label{SSVMprimal}
\end{equation}
subject to the SSVM constraints
\begin{equation}
{\frac 1 n}
	\w^\intercal
             \sum_{i=1}^n (\fpsi(x_i, y_i)-\fpsi(x_i, \bar y_i))
\geq
{\frac 1 n} \sum_{i=1}^n \Delta(y_i, \bar y_i) - \xi \label{SSVMconstraints}
\end{equation}
for any joint labeling $({\bar{y}_1,\ldots,\bar{y}_n})\in \mathcal{Y}^n$, where the loss $\Delta : \mathcal{Y} \times \mathcal{Y} \rightarrow \mathbb{R}_+$ quantifies the dissimilarity of outputs \citep{joachimsml09}. 
The slack variable $\xi$ in Equation~\eqref{SSVMprimal} represents the average training error over the whole training dataset. 
The regularization constant $C > 0$ controls the trade-off between the $\ell_2$ regularization of $\vec{w}$ and the minimization of the training error.

Despite the exponential number $|\mathcal{Y}|^n$ of constraints in the SSVM (Equation~\eqref{SSVMconstraints}), efficient algorithms, such as the cutting-plane approach~\citep{TsochantaridisJHA05,joachimsml09}, can be used to solve the quadratic program (QP) in Equation~\eqref{SSVMprimal}. 
These algorithms require efficient computation of the \emph{augmented inference max-oracle}

\begin{equation}\label{eq:augmentedMAP}
\argmax_{y\in\mathcal{Y}}{\w^\intercal \fpsi(x_i, y)} + \Delta(y_i, y).
\end{equation}

In this work, we consider the particular application of SSVM of learning the potentials of a pairwise Conditional Random Field (CRF) model \citep{Anguelov2005,SzummerKH08}. 
In this case, 
elements of $\mathcal{X}$ can be represented in terms of graphs.
An $x \in \mathcal{X}$ is associated with vertices $\mathcal{V}_x = (x^k)_{k=1\ldots |\mathcal{V}_x|}$ and edges $\mathcal{E}_x\subseteq \{ \{v_i, v_j \} | (v_i,v_j) \in \mathcal{V}_x \times \mathcal{V}_x \wedge v_i \neq v_j \}  
$. 
An output $y \in \mathcal{Y}$ is a labeling $(y^k)_k \in \mathcal{L}^{|\mathcal{V}_x|}$ of each vertex in the graph, where  $\mathcal{L}$ is the set of labels. The joint features $\fpsi(x,y)$ decompose into unary and pairwise features over the vertices and edges of the graph, such that~\citep{SzummerKH08}

\begin{dmath}\label{eq:binaryenergy}
\w^\intercal \fpsi(x, y) =
\w_{u}^\intercal \sum_{x^k \in \V_x} \fpsi^k(x^k, y^k)
+\w_{p}^\intercal \sum_{(x^k,x^l) \in \E_x} \fpsi^{k,l}(x^k, y^k,x^l,y^l).
\end{dmath}
In this setting, one can see that the MAP inference problem~\eqref{map_problem} corresponds to the traditional problem of energy minimization of the CRF \citep{Lafferty:2001:CRF:645530.655813}, each vertex having a unary energy $-\w_{u}^\intercal \, \fpsi^k(x^k, y^k)$ and each edge a pairwise energy $-\w_{p}^\intercal \fpsi^{k,l}(x^k, y^k,x^l,y^l)$.

In the following and without loss of generality, we write the joint feature maps as Kronecker products~\citep{Magnus95}:
\begin{align}\label{eq:jointfeats}
&\fpsi_u^k(x^k, y^k) = \ind(y^k) \otimes \bm{\phi}_u^k(x^k);\\ &\fpsi_p^{k,l}(x^k, y^k,x^l,y^l) = \ind(y^k) \otimes \ind(y^l) \otimes \bm{\phi}_p^{k,l}(x^k, x^l)
\end{align}
where $\ind: \mathcal{L} \to \{0,1\}^{|\lab|}$ is a one-hot encoding of the labels and $\bm{\phi}_u^k$ and $\bm{\phi}_p^{k,l}$ are unary and pairwise features associated to $x$. 
Similarly, $\w_u$ can be split along this decomposition into unary weights $\w_\alpha$ for every label $\alpha \in \lab$, and $\w_p$ into $\w_{\alpha, \beta}$ for every $\alpha, \beta \in \lab$. 

While for general losses, the augmented inference problem~\eqref{eq:augmentedMAP} can be harder to solve than the MAP inference problem~\eqref{map_problem}, a common choice is to pick a loss that decomposes over the unary energies of the graph \citep{TsochantaridisJHA05,Anguelov2005,SzummerKH08}, i.e.
\begin{equation}
    \Delta(y_i, y) = \sum_{x^k\in\V_x} \delta_k(y_i, y^k).
\end{equation}
Using such a decomposable loss, solving the augmented inference problem~\eqref{eq:augmentedMAP} is equivalent to solving a MAP inference problem~\eqref{map_problem} with modified unary energies. 
In the following section, we discuss how enforcing additional constraints on $\w$ can ensure the tractability of this MAP inference problem.

\section{Submodularity in CRFs and probably submodular constraints}\label{sec:CRFprobablySubmodular}

Solving MAP inference in pairwise CRFs is \NP-hard in general~\citep{barahona1982computational}; however, particular restrictions on the pairwise potentials give rise to efficient algorithms. 
In particular, imposing the 
submodularity condition 
on the pairwise energies of the graph, i.e.\footnote{In the following, we write $\feat_u(x^k)$ and $\feat_p(x^k,x^l)$ as a shorthand for 
$\feat_u^{k}(x^k,x^l)$ and $\feat_p^{k,l}(x^k,x^l)$.}
\begin{equation}
\begin{aligned}
\langle \w_{\alpha\alpha}&, \feat_p(x^k,x^l) \rangle + \langle \w_{\beta\beta}, \feat_p(x^k,x^l) \rangle
\\
&\geq
\langle \w_{\alpha\beta}, \feat_p(x^k,x^l) \rangle + \langle \w_{\beta\alpha}, \feat_p(x^k,x^l) \rangle \label{eq:submodcond}
\end{aligned}
\end{equation}
for every edge $\{x^k, x^l\}\in\E_x$ and every pair of labels $\alpha, \beta\in\lab$, 
leads to tractable inference: 
exact with binary labels (max-flow algorithm), or approximate with strong approximation bounds for tasks with more than two labels~--~for instance with the $\alpha-\beta$ swap algorithm of~\citet{BoykovVZ01}. 
This holds regardless of the unary energies of the graph, hence the conditions for augmented inference are the same.
Submodularity conditions~\eqref{eq:submodcond} can be enforced as constraints on the weight vector~$\w$.
In the following, we detail different sets of constraints that ensures that these inequalities are satisfied.

\paragraph{Definitely submodular constraints}  
We assume positivity of the pairwise features $\feat_p(x^k, x^l) \succcurlyeq \mathbf{0}$. The set of constraints
\begin{equation}\label{eq:restrsubmodular}
\C_1: 
\bigg\{ 
  \begin{aligned}
  \w_{\alpha\alpha} = \w_{\beta\beta} = \mathbf{0} \wedge \w_{\alpha\beta} \preccurlyeq \mathbf{0} \,\wedge\, \w_{\beta\alpha} \preccurlyeq \mathbf{0}
  \\ 
  \forall\,\alpha\!\neq\!\beta \in \lab
  \end{aligned}
\,\bigg\}
\end{equation}
introduced by~\citet{SzummerKH08}, enforces submodularity conditions~\eqref{eq:submodcond} for all inputs $x\in\mathcal{X}$.

It is immediately clear on inspection of Equations~\eqref{eq:submodcond} and~\eqref{eq:restrsubmodular} that this set of constraints is sufficient, but not necessary.  We therefore introduce the relaxed set of necessary and sufficient constraints
\begin{equation}\label{defsubmodular}
\Co: 
\bigg\{ 
  \begin{aligned}
\w_{\alpha\alpha} \succcurlyeq \mathbf{0} \,\wedge\,
\w_{\beta\beta} \succcurlyeq \mathbf{0} \,\wedge\,
\w_{\alpha\beta} \preccurlyeq \mathbf{0} \,\wedge\, \w_{\beta\alpha} \preccurlyeq \mathbf{0}
  \\ 
  \forall\,\alpha\!\neq\!\beta \in \lab
  \end{aligned}
\,\bigg\}
\end{equation}
also enforcing conditions~\eqref{eq:submodcond} to be satisfied for all inputs $x\in\mathcal{X}$.

\paragraph{Probably submodular constraints}
We now make a probabilistic argument, which we make precise in Section~\ref{sec:SampleComplexity}, that we may further relax $\mathcal{C}_2$ to enforce linear constraints on $w$ of the form in Equation~\eqref{eq:submodcond} only for the values of $\feat_p(x_i^k,x_i^l)$ observed in the training data.  We will refer to this set of constraints as
\begin{equation}\label{eq:C4}
\Cf: \!\left\{\rule{0cm}{2.5em}\right. 
   \begin{aligned}
    &\langle \w_{\alpha\alpha}, \feat_p(x_i^k,x_i^l) \rangle + \langle \w_{\beta\beta}, \feat_p(x_i^k,x_i^l) \rangle \\
 &\geq \langle \w_{\alpha\beta}, \feat_p(x_i^k,x_i^l) \rangle + \langle \w_{\beta\alpha}, \feat_p(x_i^k,x_i^l) \rangle \\
 &\hspace{3em}\forall\,\alpha\!\neq\!\beta \in \lab,\ \{x_i^k,x_i^l\}\in\E_{x_i},\ x_i\in\T
  \end{aligned}
  \left\}\rule{0cm}{2.5em}\right.
\end{equation}
enforcing these conditions for the examples $x_i\in\T$ only. 
The key insight that allows us to make this relaxation is that if a function is submodular on the training data, with high probability it will be submodular on the test data (see Section~\ref{sec:SampleComplexity}).  Furthermore, the constraints are linear in $\w$ and our optimization remains a quadratic programming problem, albeit with a large set of constraints.

As a final set of constraints, we slightly restrict $\mathcal{C}_4$ to ensure that pairwise potentials of the same label are negative (i.e.\ favored by the inference procedure), while pairwise potentials of different labels are positive (i.e.\ discouraged by the inference procedure), resulting in the set of constraints
\begin{equation}
\C_3: \!\left\{\rule{0cm}{2.5em}\right. 
   \begin{aligned}
    \hspace{-0.5em}&\langle \w_{\alpha\alpha}, \feat_p(x_i^k,x_i^l) \rangle \geq 0; \langle \w_{\beta\beta}, \feat_p(x_i^k,x_i^l) \rangle \geq 0\\
    \hspace{-0.5em}&\langle \w_{\alpha\beta}, \feat_p(x_i^k,x_i^l) \rangle \leq 0; \langle \w_{\beta\alpha}, \feat_p(x_i^k,x_i^l) \rangle \leq 0 \\
 \hspace{-0.5em}&\hspace{1em}\forall\,\alpha\!\neq\!\beta \in \lab,\ \{x_i^k,x_i^l\}\in\E_{x_i},\ x_i\in\T
  \end{aligned}
  \left\}\rule{0cm}{2.5em}\right..
\end{equation}
This specifies in a loose way prior knowledge about the role of pairwise constraints in image segmentation, while still giving sufficient model capacity to the learning algorithm.

\paragraph{Gain in model capacity.}
We have that $\mathcal{C}_1 \subset \mathcal{C}_2$, i.e. constraints $\C_1$ are tighter than constraints $\Co$. 
One might wonder if using $\Co$ effectively leads to a gain in model capacity over using constraint set $\C_1$. 
MRFs can have equivalent parametrizations, and it could be the case that any weight configuration in $\Co$ can be reparametrized as a weight configuration in $\C_1$. 
For instance, \citet[Sec.~4.1]{KolmogorovZabihPAMI2004} give a reparametrization allowing to reduce a problem with non-zero same-label pairwise potentials 
\begin{align}
-\langle \w_{\alpha\alpha} , \feat_p(x^k, x^l) \rangle  \text{ and } -\langle \w_{\beta\beta} , \feat_p(x^k, x^l)\rangle
\end{align}
to an equivalent pairwise problem where these terms are equal to 0. 
However, such a reparametrization requires that the unary features $\feat_u$ be reparametrized in order to incorporate some of the edge-dependent terms of the energy in the unary potentials. 
Assuming that the unary and pairwise features are fixed, we show that a transformation \emph{of the model weights only} is not sufficient to reduce a model in $\Co$ to a model in $\C_1$. 
In this setting, we define reparametrization as follows:
\begin{definition}

A weight vector $\w_1  \in \mathbb{R}^d$ is \emph{reparametrizable} into $\w_2 \in \mathbb{R}^d$ if both vectors lead to the same MAP solution for every element of the input space:

\begin{equation}
\argmax_{y\in\mathcal{Y}}{\w_1^\intercal \fpsi(x, y)} = \argmax_{y\in\mathcal{Y}}{\w_2^\intercal \fpsi(x, y)}
\end{equation}
for all $x\in\mathcal{X}$.
\end{definition}
Intuitively, a model optimized with constraint set $\Co$ can exploit information contained in the same-label pairwise features, if this information is not also encoded in the unary features, contrary to constraint set $\C_1$. 
This argument is made explicit in the proof of the following proposition.
\begin{proposition}

In general, elements of $\Co$ cannot be reparametrized to elements of $\C_1$.

\end{proposition}
\begin{proof}
As an example, consider a problem with two variables $x^0, x^1$, two labels $\mathcal{Y} = \{a, b\}$, and two samples in $\mathcal{S}$ with ground truth labels $y_0 = (a, a)$ and $y_1 = (b, b)$. 
Consider constant unary features $\feat_u(x^0_0) = \feat_u(x^1_0)= \feat_u(x^0_1) = \feat_u(x^1_1) = 0$ and scalar pairwise features such that $\feat_p(x_0^0, x_0^1) = -1$ and $\feat_p(x_1^0, x_1^1) = 1$. 
A weight vector such that $w_{aa} < 0$ and $w_{bb} > 0$, satisfying $\Co$, will have zero error. 
Any weight vector such that $w_{aa} = 0$ and $w_{bb} = 0$, satisfying $\C_1$, will not differentiate between the two samples, and will therefore yield non-zero error. \qed
\end{proof}
In conclusion, there is a strict increase in model capacity in general when optimizing the model weights with constraint set $\Co$ rather than $\C_1$,  
which we also validate 
in our experiments.

Since $\mathcal{C}_1 \subset \mathcal{C}_2 \subseteq \mathcal{C}_3 \subseteq \mathcal{C}_4$, we strictly increase the model capacity when we move from $\mathcal{C}_1$ to $\mathcal{C}_4$.
$\mathcal{C}_4$ may in the limit reach $\mathcal{C}_2$, but this would require a very unnatural data set to impose such strong constraints.
In all experiments, we observe that $\mathcal{C}_3$ and $\mathcal{C}_4$ are substantially larger than $\mathcal{C}_2$ and that the optimal weight vector achieved by the objective in Equation~\eqref{SSVMprimal} optimized with constraints $\mathcal{C}_3$ lies outside $\mathcal{C}_2$. 
As detailed later, $\mathcal{C}_3$ and $\mathcal{C}_4$ are empirically observed to be distinct and resulting in different optimal $\w$ (Section~\ref{sec:Results}).

\section{Sample Complexity of Probably Submodular Constraints}\label{sec:SampleComplexity}

We consider that our training images $x$ be drawn i.i.d.\ from some probability distribution $p(x)$, an assumption which is already implicit in the regularized risk minimization of the SSVM.  We therefore consider the vector valued random variable $\feat_{x}(x_i^k,x_i^l)$ where $x_i$ is drawn from $p(x)$ and $k$ and $l$ are sampled uniformly.  Our precise task is to determine whether $\mathcal{C}_3$ and $\mathcal{C}_4$ determined by the training sample results in a high probability of the scalar random variable $ \langle \w_{\alpha\alpha}, \feat_x(x_i^k,x_i^l) \rangle + \langle \w_{\beta\beta}, \feat_x(x_i^k,x_i^l) \rangle - \langle \w_{\alpha\beta}, \feat_x(x_i^k,x_i^l) \rangle - \langle \w_{\beta\alpha}, \feat_x(x_i^k,x_i^l) \rangle$ being non-ne\-ga\-ti\-ve, where $w$ satisfies $\mathcal{C}_3$ or $\mathcal{C}_4$, respectively.  We note that $\mathcal{C}_{3}$ and $\mathcal{C}_{4}$ are both convex cones as they are the intersection of half-spaces that intersect the origin. 
Here, we consider a conservative bound by noting that a convex cone enclosing the data is strictly larger than its convex hull and therefore the integral of the probability measure outside the convex cone is strictly smaller than the integral of the probability measure outside the convex hull (Figure~\ref{fig:convexHullConvexCone}).
\begin{figure}[ht]
  \centering
\includegraphics[width=0.7\linewidth]{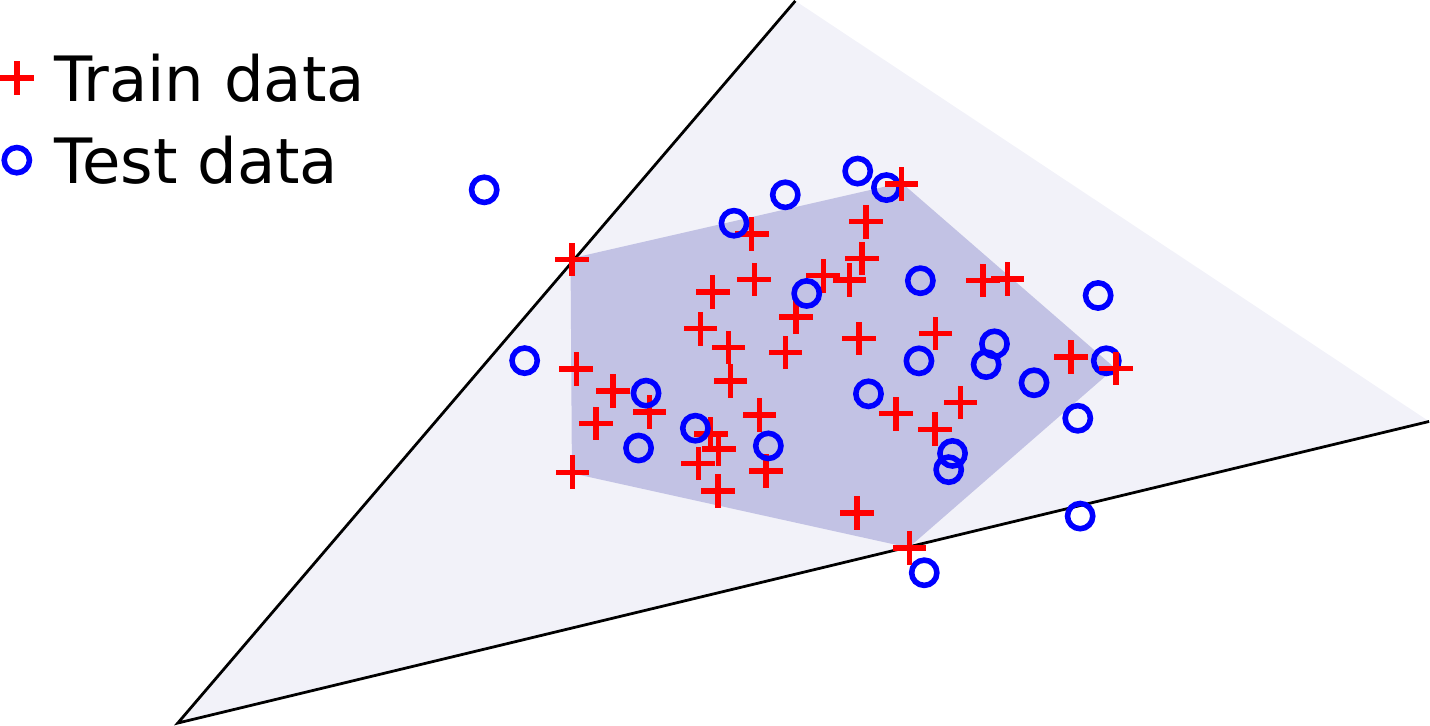} 
  \caption{The probability of a test pairwise potential being submodular can be reduced to the question of the sample complexity of learning convex cones.
To bound the probability of a random variable landing outside the convex cone defined by the training data, we use a bound on the probability of the random variable landing outside the convex hull.  This is sufficient to bound the sample complexity of probably submodular constraints as defined in Section~\ref{sec:model}.}
  \label{fig:convexHullConvexCone}
\end{figure}

\begin{proposition}\label{thm:VolConvHullMonotonic}
As $n\rightarrow \infty$, the expected probability that a test point lies in the convex hull of the training data goes to $1$. 
\end{proposition}
\begin{proof}
Let $\{X_i\}_{1\leq i < \infty}$ denote a set of random variables drawn i.i.d.\ from a probability distribution $p$.  We denote the volume of the convex hull of a sample taken with respect to measure $p$ as 
\begin{equation*}
\operatorname{vol}_p(\operatorname{conv}(X_1,\dots,X_n)).  
\end{equation*}
We have that
\begin{align}
    \operatorname{conv}(X_1,\dots,X_n) \subseteq \operatorname{conv}(X_1,\dots,X_{n+1}) \implies \\
    \mathbb{E}_{X_1,\dots,X_n \sim p} \left[ \operatorname{vol}_p(\operatorname{conv}(X_1,\dots,X_n)) \right] < \quad \nonumber \\
    \mathbb{E}_{X_1,\dots,X_{n+1} \sim p} \left[ \operatorname{vol}_p(\operatorname{conv}(X_1,\dots,X_{n+1})) \right].
\end{align}
The inequality is strict as we take the expectation of samples from $p$, and for non-trivial distributions,  $\operatorname{conv}(X_1,\dots,X_n)$ will not be equal to the support of $p$ with some probability strictly greater than zero. 

Now that we have shown monotonicity, assume that 
\begin{equation*}
\lim_{n\rightarrow \infty} \mathbb{E}_{X_1,\dots,X_n \sim p} \left[ \operatorname{vol}_p(\operatorname{conv}(X_i,\dots,X_n)) \right] < 1 .
\end{equation*}
By the definition of the volume taken with respect to measure $p$ \citep{billingsley1995probability}, this indicates that there is some portion of the space with non-zero measure that gets sampled with probability zero, a contradiction.
\qed
\end{proof}
To the best of our knowledge, finite sample bounds are not known for arbitrary distributions, but have been studied, e.g.\ for uniform distributions over polytopes \citep{DBLP:journals/corr/abs-1111-5340} in which 
\begin{dmath}
    \mathbb{E}_{X_1,\dots,X_n \sim p} \left[ 1-\operatorname{vol}_p(\operatorname{conv}(X_1,\dots,X_n)) \right] = \Omega\left(\frac{\log^{d-1} n}{n}\right) 
\end{dmath}
$d$ being the dimensionality of the polytope. 
Proposition~\ref{thm:VolConvHullMonotonic} indicates that the constraints being satisfied on the training data will result in the constraints being satisfied on the test data with high probability given sufficient data.

\section{Application settings}
In our experiments, we specialize our framework to two particular structured prediction tasks: semantic segmentation and multi-label classification.
In the following subsections, we detail the construction of the graph and features corresponding to these two tasks.

\subsection{Semantic segmentation}

We apply our approach to semantic segmentation problems.
In order to reduce the complexity of the semantic segmentation, we commonly apply a first image segmentation algorithm to segment the image or volume into superpixels (or supervoxels). 
Each image $x$ is thus decomposed into superpixels $x^k$, $k=1\ldots P$, each of which must be mapped to a label $y^k \in \mathcal{L}$ indicating which class the superpixel belongs to.

We frame this problem as a structured prediction problem: each superpixel is represented by a vertex in the input graph $\V_x$, and the edges $\E_x$ link all pairs of superpixels that are adjacent to eachother. 
The unary features $\feat_u(x^k) \in \mathbb{R}^d$ correspond to image descriptors extracted at each of the superpixels. 
The pairwise features $\feat_p(x^k, x^l)$ are chosen to be a function $\mathbf{Q}(\feat_u(x^k), \feat_u(x^k))$ of the descriptors of the two neighbouring superpixels.
In particular, we use the element-wise absolute differences between features in our experiments, such that $\phi_p(x^k, x^l)^i = |\phi_u(x^k)^i - \phi_u(x^k)^i|$ for $i = 1 \ldots d$.

\subsection{Multi-label Classification}\label{sec:MultiLabelClassificationApplicationSetting}

We evaluate our approach on multi-label classification tasks. 
For this problem, multiple classes can be assigned to each example. 
Such tasks constitute non-trivial structured learning problems, and are good test-bed for the study of structured learning algorithms with inexact inference~\citep{Finley2008TrainingSS}. 

The classification problem consists in learning a function which maps inputs represented as a vector of $d$-dimensional attributes $\vec{x} \in \mathbb{R}^d$ to binary vectors $\vec{y} \in \{0, 1\}^{|\mathcal{C}|}$, indicating the presence or absence of each class, given a set of classes $\mathcal{C}$. This reduces the problem to a binary structured prediction problem. 
To each input $\vec{x} \in \mathbb{R}^d$, we associate an indirected fully-connected graph with $|\mathcal{C}|$ vertices, where each binary output label $y^k$ corresponds to a vertex $x^k \in \mathcal{V}_x$.

We assign unary and pairwise features specific to each vertex $v^k$ and edges $\{x^k, x^l\} \in \mathcal{E}_x$ in the graph by setting
\begin{align}
\feat_u(x^k) &= \ind(x^k) \otimes \vec{x};\\
\feat_p(x^k, x^l) &= \bm{2}(\{x^k, x^l\}) \otimes \vec{R}(\vec{x})\label{eqmultilabelR}
\end{align}
where $\ind: \mathcal{V}_x \to \{0,1\}^{|\mathcal{C}|}$ is a one-hot encoding of the vertices, $\bm{2}: \mathcal{E}_x \to \{0,1\}^{|\mathcal{E}_x|}$ a one-hot encoding of the edges, and $\vec{R}: \mathbb{R}^d \rightarrow \mathbb{R}^e$ extracts an edge feature vector from $\vec{x}$. 
This specifies the unary and pairwise features and allows us to use our probably submodular framework. 
The resulting SSVM model has unary weights $\w_u$ of dimension $2 \,\lvert\mathcal{C}\rvert\, d$, and pairwise weights $\w_p$ of dimension $2\, \lvert\mathcal{C}\rvert \,(\lvert\mathcal{C}\rvert - 1)\, e$ for a fully-connected graph structure.

\section{Efficient constraint generation}\label{sec:constgen}
The tractability constraints in sets $\Co, \Cf$ can be written as hard linear constraints $\mathbf{c}^\intercal \w \geq \mathbf{0}$. 
As such, we can incorporate them in the QP optimization \eqref{SSVMprimal}. 
However, $\Cf$ comprises $(|\lab|\cdot(|\lab| - 1)/2)\cdot|\E|\cdot|\T|$ constraints; even for moderately sized binary segmentation tasks with limited connectivity on small datasets, this large amount cannot be handled by QP solvers. We address this problem in a cut approach; the most violated constraints are iteratively added to the $\w$-update (QP solver) subroutine of the SSVM until all constraints are satisfied, leading experimentally to a small, manageable, number of constraints added to the QP at any learning iteration.

Noting $\fp$ the dimension of the pairwise features, let $\matr{P}$ be the $|\E|\cdot|\T|\times\fp$ matrix of all pairwise feature vectors in the training data and $\matr{B}$ the $|\lab| (|\lab|-1)/2\times |\lab|^2$ matrix of rows
\begin{multline}
(\ind(\alpha) \otimes \ind(\alpha))^{\intercal} + (\ind(\beta) \otimes \ind(\beta))^{\intercal} \\
- (\ind(\alpha) \otimes \ind(\beta))^{\intercal}   - (\ind(\beta) \otimes \ind(\alpha))^{\intercal} 
\end{multline}
for all labels $\alpha \neq \beta$. The constraints in $\Cf$ take the form $(\matr{B}\otimes \matr{P})\w_p \geq 0$. Because of the large number of constraints, computing the constraints margins after each $\w-update$ takes a significant amount of computation time. We detail in the following different ways to reduce the computational impact of this computation.

\paragraph{Tensor factorization} The complexity of computing the constraint margins $(\matr{B}\otimes \matr{P})\w_p$ can be reduced by observing that 
\begin{equation}\label{eq:ConstraintMarginsSpeedupTensorFactor}
(\matr{B}\otimes \matr{P})\w_p = \vect{\matr{P}\matr{\tilde{\w}}_p \matr{B}^\intercal}
\end{equation}
with $\matr{\tilde{\w}}_p$ a matrix constructed such that $\vect{\matr{\tilde{\w}}_p} = \w_p$. The computation of the right-hand side $\matr{V} \coloneqq \matr{P}(\matr{\tilde{\w}}_p \matr{B}^\intercal)$ saves a factor of $|\lab|^2$ operations for computing all the constraint margins.
The resulting $\w$-update subroutine of the SSVM is presented in Algorithm~\ref{alg:QP_naive}.

\SetKwFor{Loop}{Loop}{}{EndLoop}

\IncMargin{1em}
\begin{algorithm}[ht]

\SetKwInOut{Input}{input}\SetKwInOut{Output}{output}
\SetKwComment{Comment}{\quad// }{}
\SetCommentSty{itshape}
\newlength{\commentsnaiveone}
\newlength{\commentsnaivetwo}
\setlength{\commentsnaiveone}{9.5em}
\setlength{\commentsnaivetwo}{14em}

\Input{SSVM constraints $\C^{(t)}$ at iteration $t$; 
constraint matrices $\matr{P}, \matr{B}$; 
current $\w$ 
}
\Output{optimal $\w^*$ satisfying $\C^{(t)}$ and submodular constraints}
\BlankLine

\Loop{\label{alg:QP_naive:inner_loop}}{
$(\w^*, \xi^*) \gets \argmin_{(\w, \xi)} \ \{\lVert \w \lVert^2/2 + C \xi\} \enskip $\\ s.t. $(\w,\xi)\in\C^{(t)}$\Comment*[r]{\makebox[\commentsnaiveone][l]{QP solver}}

$\matr{W} \gets \matr{\tilde{\w}}_p \matr{B}^\intercal$\\
$\matr{V} \gets \matr{P}\matr{W}$\\
$(i, j) \gets \argmin{\matr{V}}$\\
\lIf(\Comment*[f]{\makebox[\commentsnaivetwo][l]{all bounds positive}}){$V_{i, j} \geq 0$}{\Return $\w^*$}
\lElse{$\C^t = \C^t \cup \{\text{constraint $(\mathbf{b}_j \otimes \mathbf{p}_i)\w_p \geq 0$}\}$ \Comment*[f]{\makebox[\commentsnaivetwo][l]{most violated constraint}}
}
}
\caption{$\w$-update subroutine of probably submodular SSVM\label{alg:QP_naive}}
\end{algorithm}\DecMargin{1em}

\subsection{Delayed constraint generation}\label{sec:DelayedConstraintGenerationNIPSWorkshop} 
Even with this acceleration, computing the $\lvert\E\rvert \lvert\T\rvert\times|\lab|(|\lab|-1)/2$ matrix $\matr{V}$ still requires $\mathcal{O}(|\E|\cdot|\T|\cdot|\lab|^2\cdot d)$ operations. On a small-sized problem with $10^3$ edges, $200$ images, $2$ labels and pairwise features of dimension $500$, this results in $400$ million floating point operations for updating the hard constraints margins after each update of the weight vector, which significantly impacts the learning time -- as illustrated by our experiments.

To address this issue, we use a delayed constraint generation approach. The key observation is that in later learning iterations, the optimal weight vector $\w$ does not change drastically. Constraints corresponding to a high positive margin in $\matr{V}$ are therefore likely to stay enforced after updating $\w$. Formally, for each probably submodular constraint $c: \mathbf{c}^\intercal \w \geq 0$, we introduce a lower bound on the margin $l_c \leq \mathbf{c}^\intercal \w$. After a weight update $\w \rightarrow \w'$, we have
\begin{equation}\label{cauchyschwartz}
 \mathbf{c}^\intercal \w'
 = \mathbf{c}^\intercal \w + \langle \w' - \w, \mathbf{c} \rangle 
 \geq l_c-\|\w'-\w\| \cdot \|\mathbf{c}\|;
\end{equation}
by application of the Cauchy-Schwartz inequality. Therefore the update 
$
l_c \rightarrow l_c' = l_c-\|\w'-\w\| \cdot \|\mathbf{c}\|
$ 
yields a correct new lower bound. We can safely save computations by avoiding the updating of constraint margins that are lower-bounded by a positive value.

As before, we write the operations in matrix form. We store the norm of all constraints, and the margin of lower bounds (initialized to $-\infty$), in two matrices $\matr{N}$ and $\matr{L}$ of same size as $\matr{V}$. 
By storing the results of margin computations in $\matr{L}$, raising the bound to the actual value, we avoid referring to $\matr{V}$ altogether. Algorithm~\ref{alg:QP_fast} presents the resulting algorithm integrated in the $\w$-update subroutine, called at each iteration $t$ of the SSVM. 

\IncMargin{1em}
\begin{algorithm}[ht]
\SetKwInOut{Input}{input}\SetKwInOut{Output}{output}
\SetKwComment{Comment}{\quad// }{}
\SetCommentSty{itshape}
\newlength{\commentsone}
\newlength{\commentstwo}
\setlength{\commentsone}{7em}
\setlength{\commentstwo}{7em}
\Input{SSVM constraints $\C^{(t)}$ at iteration $t$; 
constraint matrices $\matr{P}, \matr{B}, \matr{N}$; 
current $\w$, current bounds $\matr{L}$ 
}
\Output{optimal $\w^*$ satisfying $\C^{(t)}$ and submodular constraints; new bounds $\matr{L}$}
\BlankLine

\Loop{\label{alg:QP_fast:inner_loop}}{
$(\w^*, \xi^*) \gets \argmin_{(\w, \xi)} \ \{\lVert \w \lVert^2/2 + C \xi\}$ \\ \hspace{3em} s.t. $(\w,\xi)\in\C^{(t)}$\Comment*[r]{\makebox[\commentsone][l]{QP solver}}

$\matr{L} \gets \matr{L} - \| \w^* - \w\| \cdot \matr{N}$ \Comment*[r]{\makebox[\commentstwo][l]{update bounds}}
$\matr{W} \gets \matr{\tilde{\w}}_p \matr{B}^\intercal$\\
\For{(i,j) such that $L_{i,j}\leq 0$}{
$L_{i,j} \gets \sum_k P_{i,k} W_{k,j}$ \Comment*[r]{\makebox[\commentstwo][l]{compute margins}}
}
\BlankLine
$(i, j) \gets \argmin{\matr{L}}$\\
\lIf(\Comment*[f]{\makebox[\commentstwo][l]{all bounds positive}}){$L_{i, j} \geq 0$}{\Return $\w^*, \mathbf{L}$}
\lElse{$\C^t = \C^t \cup \{\text{constraint $(\mathbf{b}_j \otimes \mathbf{p}_i)\w_p \geq 0$}\}$ \Comment*[f]{\makebox[\commentstwo][l]{most violated constraint}}
}
}
\caption{Accelerated $\w$-update subroutine of probably submodular SSVM\label{alg:QP_fast}}
\end{algorithm}\DecMargin{1em}

\paragraph{Pretraining} In initial iterations of the learning procedure, $\w$ changes significantly and most of the constraints have to be recomputed. To mitigate this, we can use a two-stage learning. First, the weights and dual variables of the SSVM are trained until convergence with no submodular constraints -- resulting in an inexact truncated graph-cut inference. Second, we enfore submodularity with the above approach, with exact inference. The SSVM converges to the same global optimum, but the pretraining warm-starts the exact learning closer to convergence.

\paragraph{Constraint minibatches} The 2-stage \emph{pretraining} strategy can be generalized to an $n$-stage learning approach, where batches of constraints are added sequentially, in order to balance the computation time between convergence to the objective of the SSVM and computation of margins. We experiment with this approach in the experimental section.

\section{Experiments and results}\label{sec:Results}

\subsection{Segmentation on TU Darmstadt dataset\label{sec:toyseg}}
We evaluate the computational gains of our method on a 10-fold cross validation of the 111 images of 
the TU Darmstadt cows dataset provided by~\cite{tucows}.
The images are first oversegmented into $\sim 500$ SLIC superpixels~\cite{slicsuperpix} with a compactness parameter $m=20$ and a prior smoothing with a gaussian kernel of width $\sigma=2.0$ pixels. We use 3-channel color histograms with $5^3 = 125$ bins as unary features $\feat_u(x^k)$ for every superpixel $k$, and absolute differences of histograms as pairwise features for adjacent superpixels. 

Our inference uses graph-cuts optimization \citep{BoykovVZ01} which is exact in the case of submodular potentials in this binary setting. The dynamic introduction of submodular constraints before each step of loss-augmented inference, as described in Section~\ref{sec:constgen}, ensures that the CRF potentials remain submodular over the training set throughout the entire training, and therefore, the exactness of the training procedure. 
In our probably submodular framework, non-submodular potentials can arise at test-time; we truncate these potentials to recover submodularity prior to the graph-cuts procedure, as described in \cite{rother2005digital}. 

In line with standard practice in segmentation, we report results with respect to two metrics.
The \emph{global} metric counts superpixel-wise accuracy, while the \emph{average} metric counts average per-class superpixel-wise accuracies. The latter metric is more informative, as the background is typically much more prevalent than the foreground. 
We have employed both variants in the construction of the structured output loss function, $\Delta(y_i, y)$, and have trained different models that have optimized each. 

For each method and loss metric an initial training with $1/5$ fraction of the training data set aside as validation data to set the SSVM regularization parameter $C$. 
This is done in order to remove the dependence on the regularization parameter for comparing the performance between methods. 
$C$ is chosen as the best performing one on the validation set among five values logarithmically spaced between $0.1$ and $10$. We set the stopping criterion of the SSVM as a threshold of $\mathit{tol} = 0.001$ on the value of the dual gap relative to the objective. 
For this criterion all methods converged within a similar number of cutting plane iterations $126\pm21$. 
\begin{figure}[ht]
    \centering
    \subfigure[Accuracies for models trained to optimize the sum of pixel errors over all training images]{
      \includegraphics[width=0.9\linewidth]{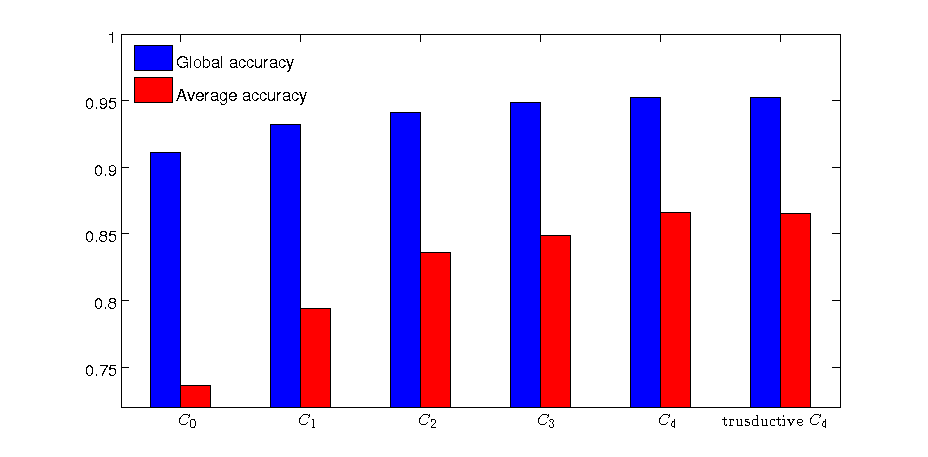}
    }
    \subfigure[Accuracies for models trained to optimize the average per-class pixel accuracies.]{
      \includegraphics[width=0.9\linewidth]{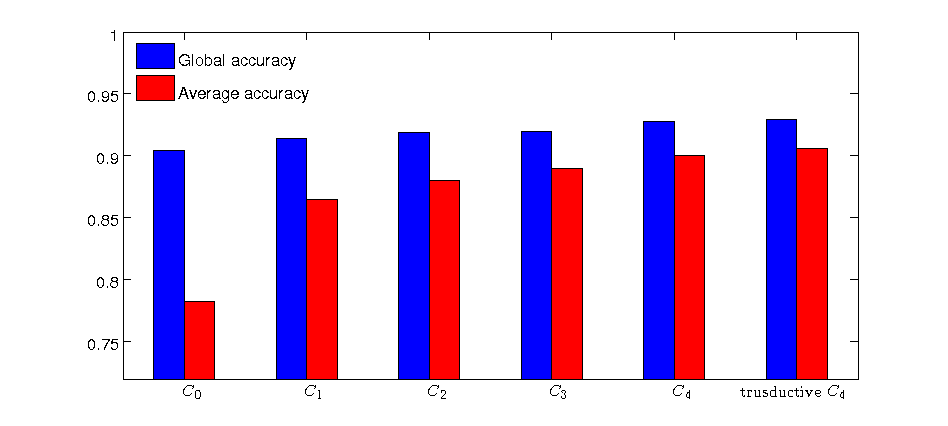}
    }
      \caption{Comparison of results of the binary segmentation of the TU Darmstadt Database of cows. Models have increased capacity as the plot moves from left to right.}
      \label{plot:accuracy}
\end{figure}

\begin{figure*}[ht]
\centering
\subfigure[Original image.]{
      \includegraphics[width=0.23\textwidth]{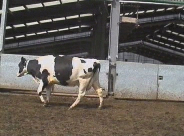}
}
\subfigure[Ground truth.]{
      \includegraphics[width=0.23\textwidth]{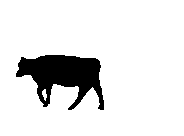}
}
\subfigure[SVM ($\C_0$).]{
      \includegraphics[width=0.23\textwidth]{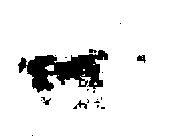}
}
\subfigure[\citet{SzummerKH08} ($\C_1$).]{
      \includegraphics[width=0.23\textwidth]{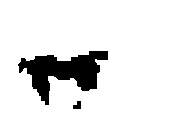}
}
\subfigure[Necessary and sufficient definitely submodular ($\Co$).]{
      \includegraphics[width=0.23\textwidth]{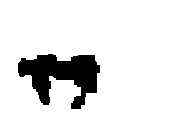}
}
\subfigure[Probably submodular ($\C_3$).]{
      \includegraphics[width=0.23\textwidth]{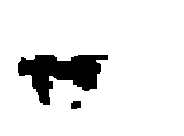}
}
\subfigure[Probably submodular ($\Cf$).]{
      \includegraphics[width=0.23\textwidth]{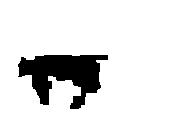}
}
\subfigure[Probably submodular with transductive constraints (\Ctrans{}).]{
      \includegraphics[width=0.23\textwidth]{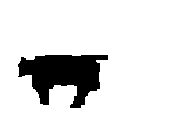}
}
\caption{Example segmentations from methods trained to optimize each of the constraint sets considered here.  As we move from (a) to (e) we increase model capacity and substantially increase the accuracy of the resulting segmentation.
.}\label{fig:ExampleSegmentations}
\end{figure*}

Results are presented in Figure~\ref{plot:accuracy}. In addition to the four constraint sets described in Section~\ref{sec:model}, we consider two additional constraint sets in our experiments. The constraint set 
\begin{equation} \label{eq:constSVM}
\C_0 = \{ \w \ | \ \w_{\alpha\alpha} = \w_{\beta\beta} = \w_{\alpha\beta} = \w_{\beta\alpha} = \mathbf{0} \}
\end{equation}
represents a simple SVM, with no effective pairwise potentials, which we cast as a special case of our learning framework. The constraint set denoted as \Ctrans{} corresponds to a variant of the probably submodular model where, in addition to the submodularity constraints on the training set, the training procedure enforces that the pairwise potentials of the test set (not observing the test labels) are guaranteed to be submodular.  This setting falls under the general definition of transductive learning, in which the test points are known during the training procedure \citep[Chapter 8]{Vapnik1998}.

We have that $\C_0 \subset \C_1 \subset \Co \subseteq \C_3 \subseteq \Cf$. As expected, this gain in model expressiveness translates into a gain of performance as we go from smaller to bigger acceptable optimization domains. 
These performance gains have been verified to be statistically significant. 
A t-test yields a $p$-value less than $10^{-5}$ for a comparison of the methods resulting from optimizing the SSVM subject to $\w \in \C_1$ vs.\ $\w\in \Cf$ both for the \emph{global} and \emph{average} metrics.

Table~\ref{constraints} presents the number of active hard constraints for every method. Hard constraints that are active at convergence indicate that the data term in the optimization objective is pushing the vector $\w$ towards a solution outside of the constraint set $\mathcal{C}_i$ of the optimization domain. We see that as we go to broader optimization domains, from $\C_0$ to $\C_f$, the hard constraints tend to be more often inactive, therefore putting less stringent constraints on the learned model.

\begin{table}[ht]
\centering
\begin{tabular}{@{}llcc@{}}
\toprule
\hspace{1em} & & \bfseries \shortstack{Active  \\ constraints}  & \bfseries \shortstack{All  \\ constraints} \\ \midrule
\multicolumn{3}{l}{\small \textsc{Definitely submodular models}} \\ 
&$\C_0$ & 326 & 415 \\
&$\C_1$ & 301 & 426 \\ 
&$\Co$ & 256 & 405 \\ 
&\Ctrans & 126 & 814 \\ 
\midrule
\multicolumn{3}{l}{\small \textsc{Probably submodular models}} \\ 
&$\C_3$ & 999 & 1609 \\
&$\Cf$ & 116 & 699 \\ 
\bottomrule
\end{tabular}
\caption{\emph{Active constraints}: number of active hard constraints after convergence. \emph{All constraints}: number of unique hard constraints introduced in the QP-solver at any point over the course of the SSVM optimization.}
\label{constraints}
\end{table}

Probably submodular models trade off inference error with model expressivity. Table~\ref{tradeoff} shows the percentage of pairwise constraints that are non-submodular in the test set.
Figure~\ref{fig:ExampleSegmentations} gives examples of segmentations predicted by the different constraint sets, including the method of~\citet{SzummerKH08} corresponding to the set $\C_1$, and the probably submodular constraint sets $\C_3$ and $\Cf$.

\begin{table}[ht]
\centering
\begin{tabular}{@{}ccc@{}}
\toprule
       & \multicolumn{2}{c}{\textbf{Non-submodular potentials}} \\
    \cmidrule(lr){1-2}
    \cmidrule(lr){2-3}
 & cows dataset & feline retinal dataset \\
       \midrule
$\C_3$ & $5\cdot 10^{-2}$ &   $4\cdot 10^{-4}$                            \\
$\Cf$  & 0 &  $1\cdot 10^{-2}$      \\ \bottomrule
\end{tabular}
    \caption{Fraction of non-submodular potentials measured on the test set for probably submodular models.
For definitely submodular models, and for the transductive probably submodular set \Ctrans{}, this number is always zero.}
    \label{tradeoff}
\end{table}

\begin{table*}[ht]\small
\centering
\caption{Training efficiency gains with/without delayed constraint generation and 2-stage weights pretraining}
\label{tbl:efficiency}
\begin{tabular}{@{}lcccc@{}}
\toprule
method       & \# margins computed & SSVM iterations & total constraint gen. time & total training time \\ \midrule
def. submodular &       &   $296$   &         &     $607\,\mathrm{s}$    \\
1-pass, full & $102.5\cdot 10^6$  & $581$ &   $400\,\mathrm{s}$      &    $593\,\mathrm{s}$    \\
1-pass, delayed    &    $\hphantom{1}67.9 \cdot 10^6$     &    $581$      &  $329\,\mathrm{s}$        &    $652\,\mathrm{s}$    \\

2-pass delayed          &    $\hphantom{10}6.5\cdot 10^{6}$   &   658           &       $\hphantom{1}41\,\mathrm{s}$     &   $555\,\mathrm{s}$      \\ \bottomrule
\end{tabular}
\end{table*}

\begin{figure*}[ht]
\centering
\includegraphics[width=0.98\textwidth]{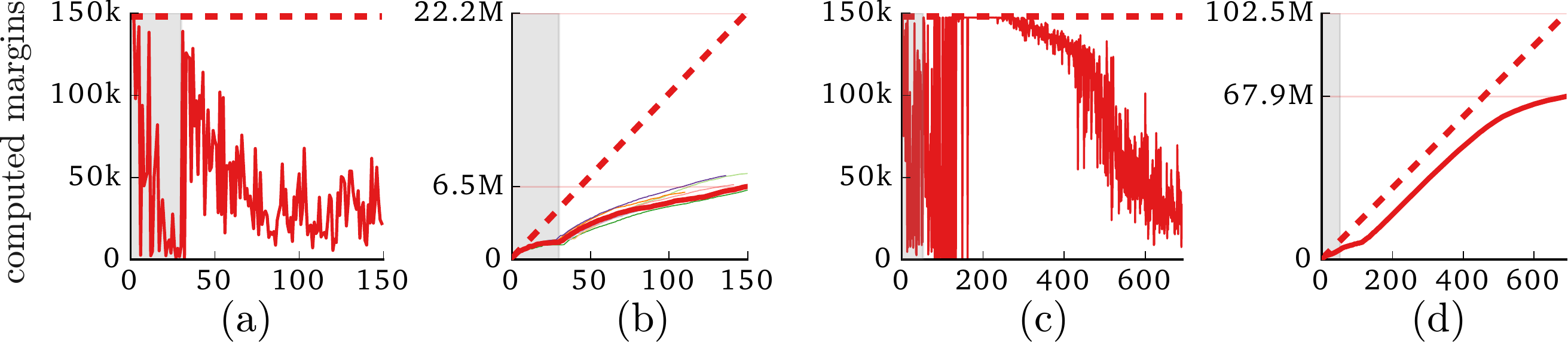}\vspace{-0.6em}
\caption{Number of constraint margins computed with (solid) or without (dashed) delayed constraint generation.  The horizontal axis is the number of cutting plane iterations (i.e.\ the numer of times passing through the outer loop of Algorithms~\ref{alg:QP_naive} or~\ref{alg:QP_fast}). (a): per call to the QP solver, (b): running total. (a) shows only one fold of the dataset, (b) also shows other folds, which have similar behavior. Shaded area: first SSVM iteration. (c), (d): margins computation without inexact pretraining: in this case, the delayed constraints approach saves a substantial amount of computation.~\label{compute_gains}}
\end{figure*}

We do not observe a difference in performance between $\Cf$ and \Ctrans{}, which corroborates our observation made in section~\ref{sec:SampleComplexity} that enforcing the constraints on the training set is enough to enforce with high probability the submodularity of the potentials at test-time. 
This also shows that the inexactness of the truncated graph-cuts procedure, due to the submodular edges arising at test-time after probably submodular training with $\Cf$, plays negligible role, as is expected given the small portion of effective non-submodular potentials on the test-set observed in Table~\ref{tradeoff}.

\paragraph{Evaluation of the constraint generation efficiency}

Table~\ref{tbl:efficiency} shows the improvements in computational efficiency of our delayed constraint generation scheme and the 2-stage training, for one fold of the data. 
Figure~\ref{compute_gains} shows the number of computed constraint margins. Combining inexact pretraining and delayed constraint generation limits the number of computed margins. In the first SSVM iteration, many hard constraints have to be added 
to make $\w$ satisfy the constraints $\Cf$; in subsequent iterations, the number of added constraints per iteration becomes small (1 or 2).

\begin{figure*}[ht]
\centering
\subfigure[Ground truth]{
      \includegraphics[width=0.23\textwidth]{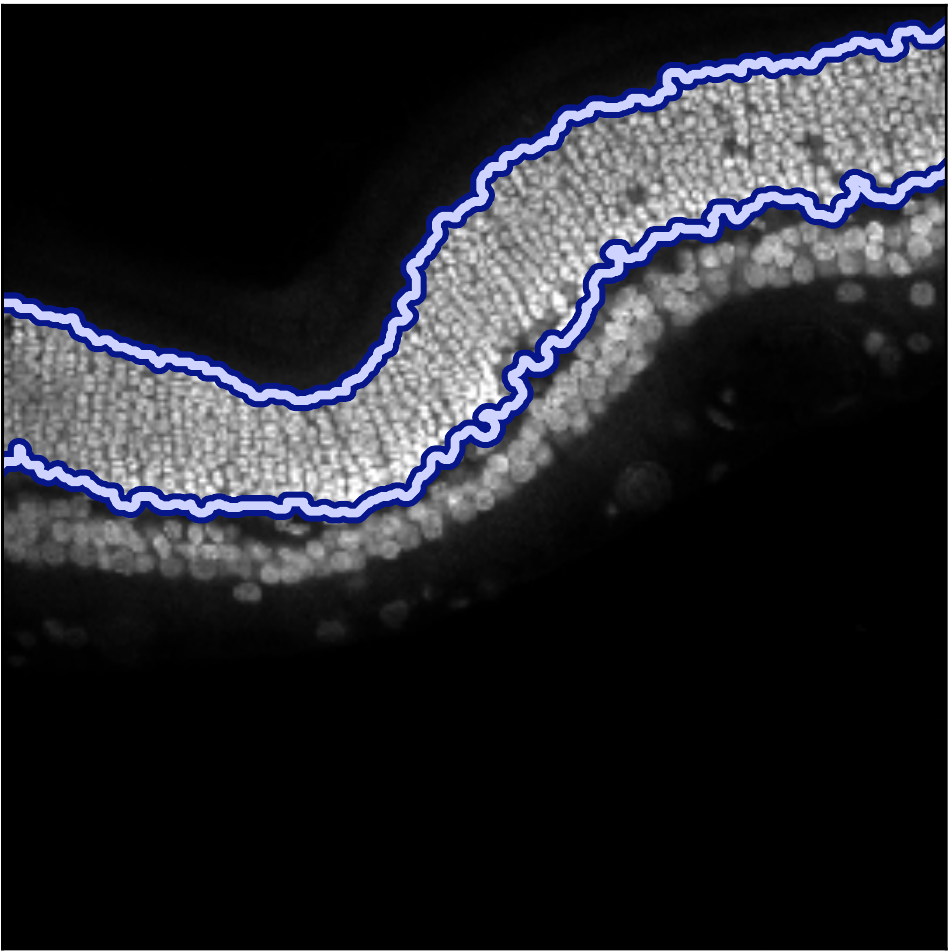}
}
\subfigure[$\C_0$]{
      \includegraphics[width=0.23\textwidth]{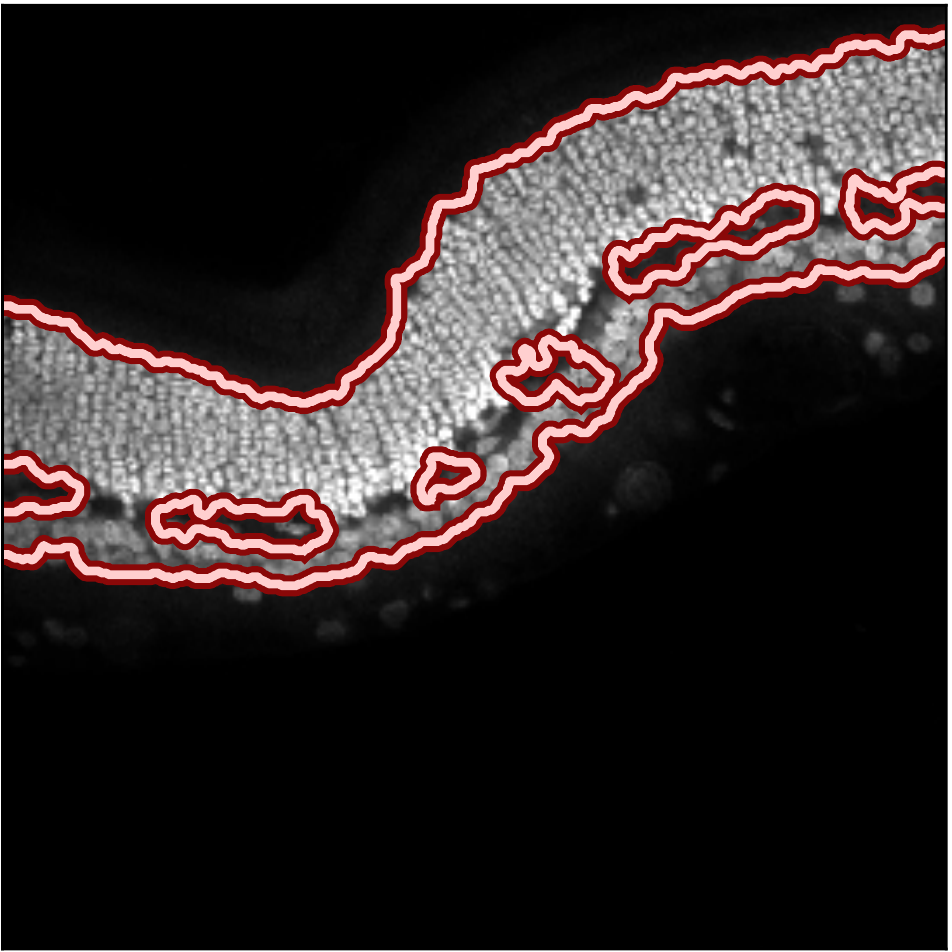}
}
\subfigure[$\C_1$]{
      \includegraphics[width=0.23\textwidth]{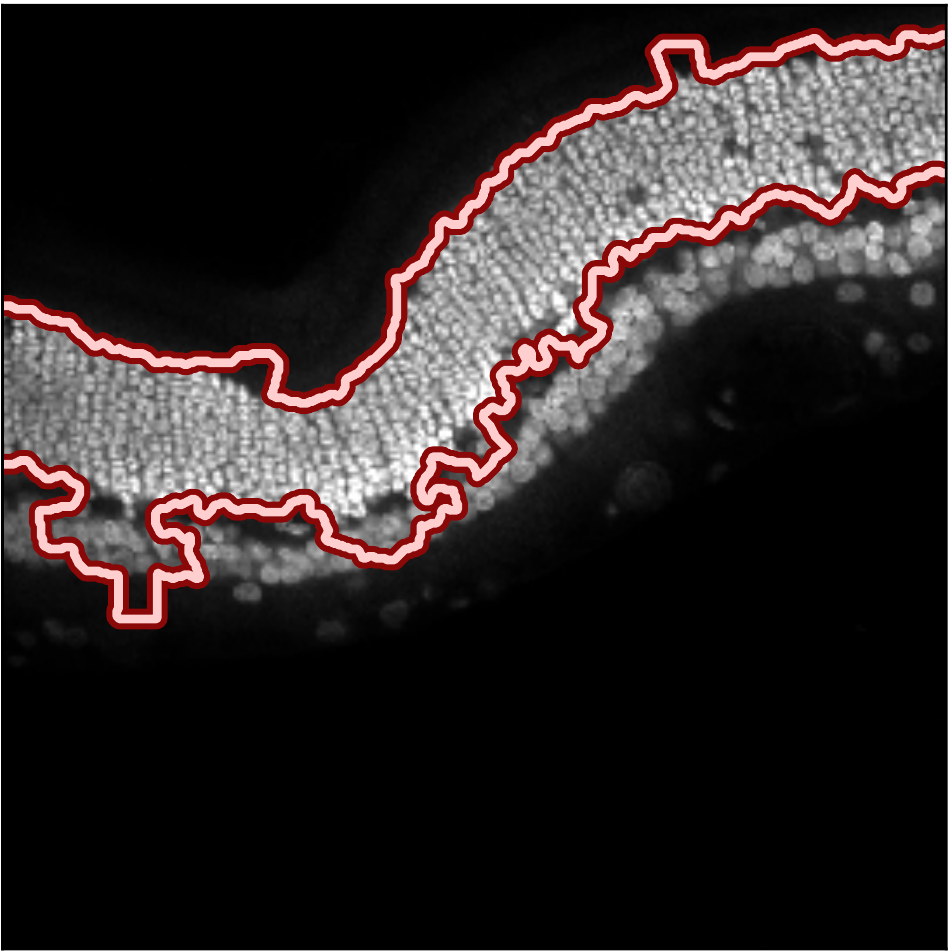}
}
\subfigure[$\Co$]{
      \includegraphics[width=0.23\textwidth]{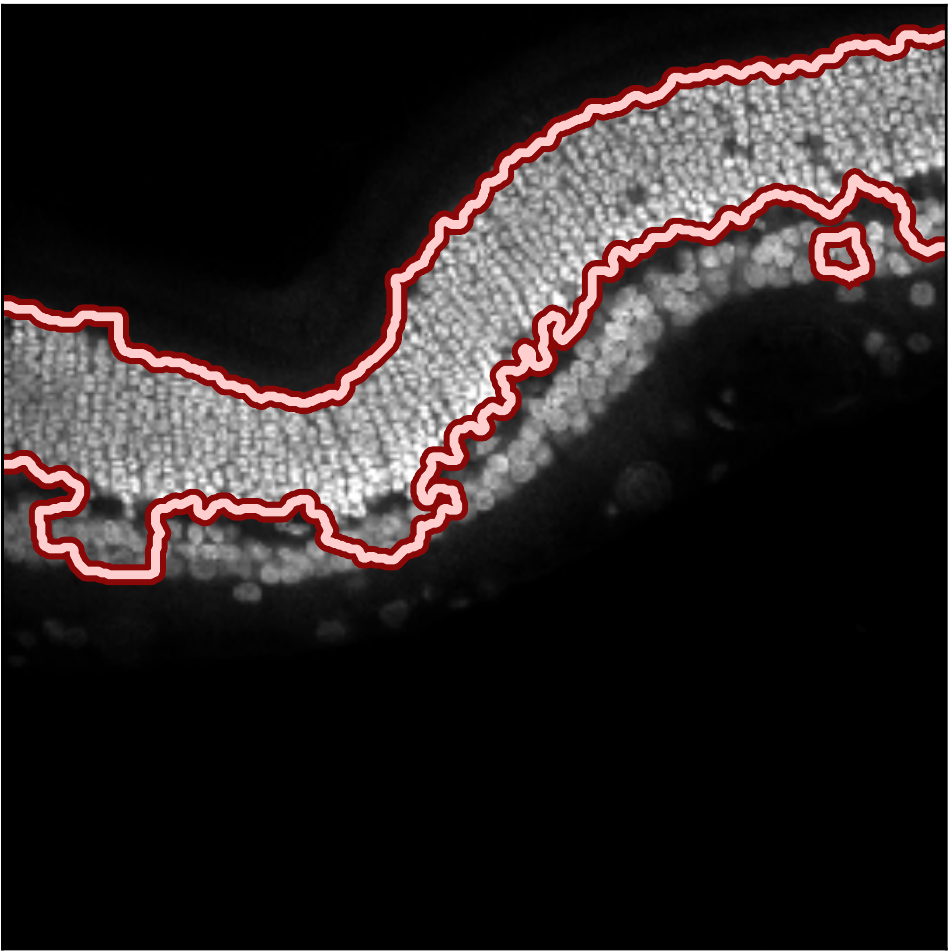}
}
\subfigure[$\C_3$]{
      \includegraphics[width=0.23\textwidth]{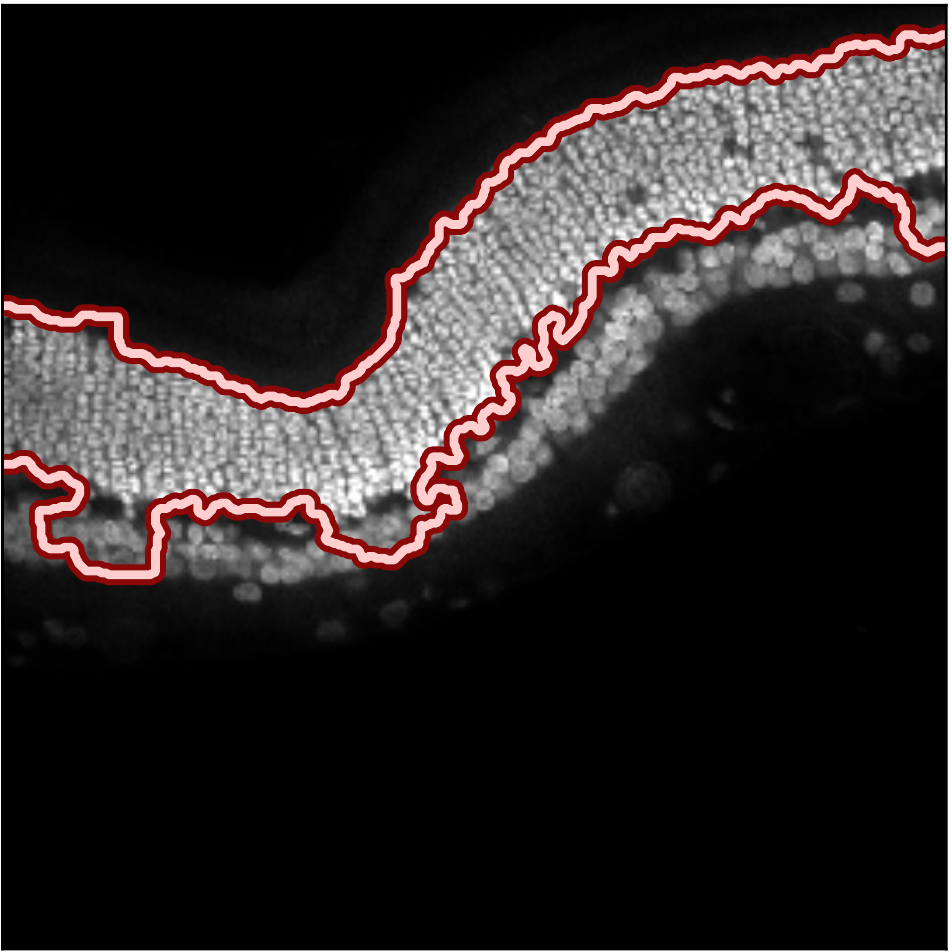}
}
\subfigure[$\Cf$]{
      \includegraphics[width=0.23\textwidth]{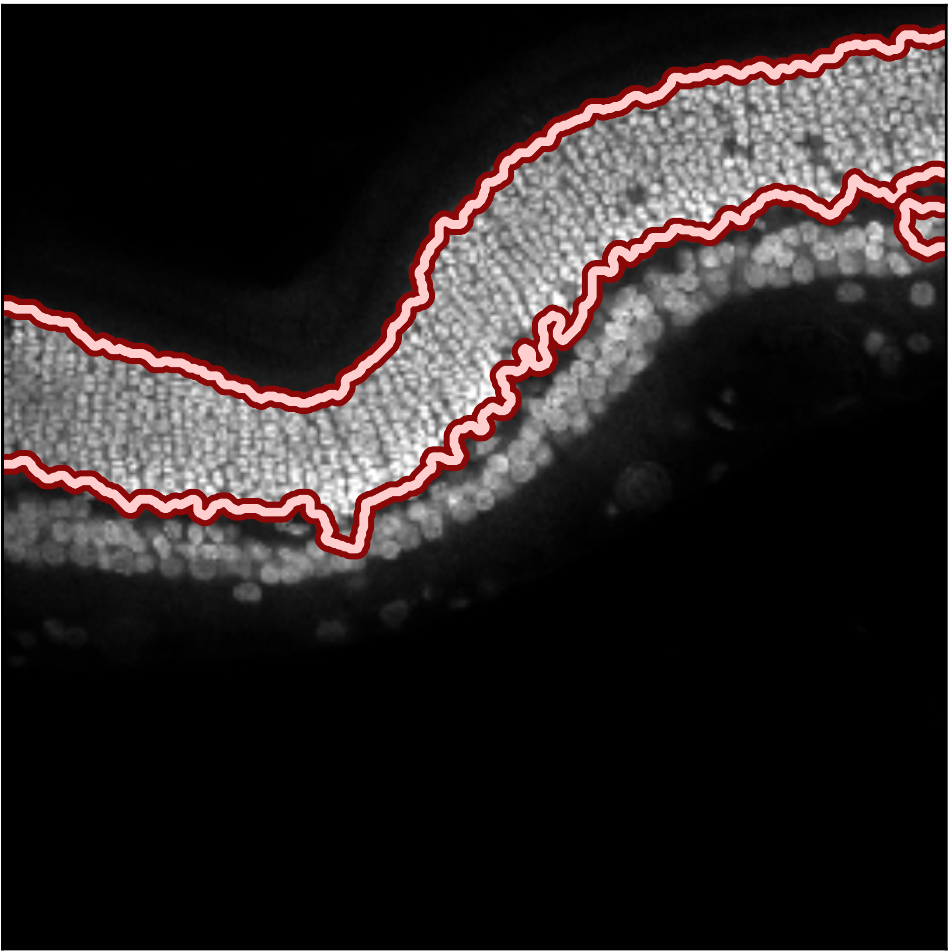}
}
\subfigure[\Ctrans{}]{
      \includegraphics[width=0.23\textwidth]{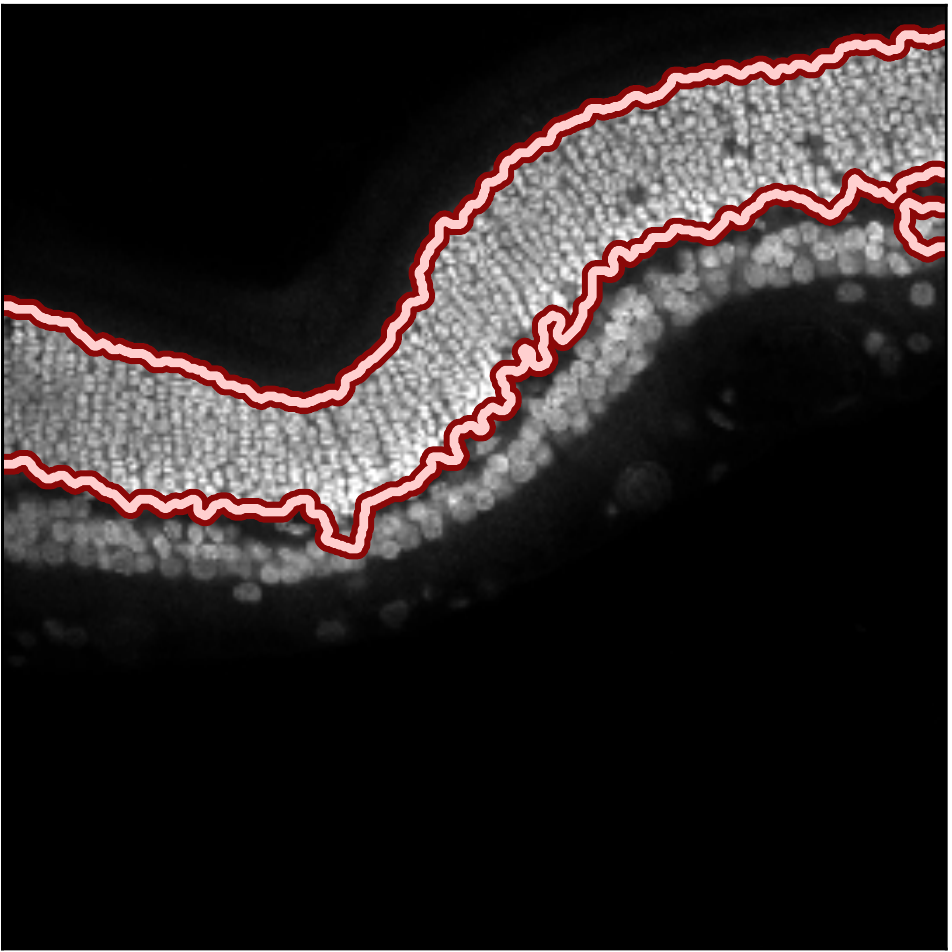}
}
\caption{Contours of segmentations obtaining using the different constraint sets on a feline retinal image. The use of larger constraint sets leads to a better performance of the model; moreover the transductive probably submodular constraint set \Ctrans{} does not perform better than the probably submodular set $\Cf$.}\label{fig:FelineSegmentations}
\end{figure*}

\begin{table*}[ht]
\centering
\caption{Test performances and Jaccard index of the segmentations obtained on the UCSB retinal dataset. 
The values are averaged across the $40$ random folds.
VOC is the VOC score (or mean IoU).}
\label{table:feline}
\begin{tabular}{@{}lcccccccc@{}}
\toprule
 & $\C_0$ & $\C_1$ & $\Co$ & $\C_3$ & $\Cf$ & \Ctrans{} 
 & QPBO \\ 
\midrule
$1-\Delta_{\text{avg}} (\%)$ &   93.1     &   93.5  &  95.7 &     95.7   &  \textbf{96.6}  &  \textbf{96.6}  & 92.4 \\
IoU (\%) &  72.8  &  78.5  &  85.0 &  85.2  &  \textbf{87.2}  &   \textbf{87.2}  & 79.7 \\
VOC (\%) &  73.3  &  80.0  &  86.4 &  86.7  &  \textbf{88.5}  &   88.4 & 81.7
\\\bottomrule
\end{tabular}
\end{table*}

\subsection{Segmentation photon receptor cells in 2D retinal images}\label{sec:ApplicationPhotonReceptorCells2Dretinal}
We evaluate our approach on the publicly available UCSB retinal dataset, which has 50 greyscale laser scanning confocal images of normal and 3-day detached feline retinas \citep{gelasca2008evaluation}.
As in subsection~\ref{sec:toyseg}, we first oversegment the images into $\sim 600$ superpixels using the SLIC algorithm.
As in a previous work by \citet{lucchi2015learning} that used structured prediction on this dataset, we use a concatenation of 10-bin intensity histograms and $8\times 8$ grey level co-occurrence matrix (GLCM) with one-pixel displacement \citep{haralick1973textural} as unary features for each of the superpixels. This yields a 74-dimensional unary feature vector.
We set the pairwise features between adjacent superpixels as the vector of absolute difference between the corresponding unary features. 
We train the different methods using the per-class averaged accuracy loss $\Delta_{\text{avg}}$. 
We report the results over $40$ random folds, each fold containing $40$ training images $10$ test images.
For each method, the regularization parameter $C$ is chosen by cross-validation on the training set of each fold. 
We set the relative tolerance of the cutting-plane SSVM to $\mathit{tol} = 10^{-4}$.
We report the resulting accuracy using the class-averaged accuracy over all superpixels in the test set -- which is the training objective and therefore accurately represents the gain in capacity of the probably submodular framework. We also report the mean per-image pixel-wise Jaccard index -- or intersection-over-union metric
\begin{equation}\label{def_jaccard}
\text{IoU} = \frac{
\text{True Positive}}{
\text{True Pos.} + \text{False Pos.} + \text{False Negative}}
\end{equation}
commonly used in image segmentation.
additionally, we report the VOC score, which is the mean between the IoU of the foreground class and the IoU of the background class, and was used in \cite{lucchi2015learning} in place of the IoU of the foreground class alone~\citep{Cheng2017VolumeSU,Casser2018FastMS}. 
These results are summarized in Table~\ref{table:feline}, averaging the performance across folds on the test set. 
We observe that the intermediate probably constraint set $\C_3$ does not provide significant gains over the definitively submodular set $\Co$, indicating that this set of constraints effectively degenerates to the definitively submodular set of constraints. 
The full-fledged probably submodular approach under the constraint set $\Cf$ does nevertheless lead to performance gains over $\Co$, significant under a paired Student's t-test with a $5\%$ acceptance level. 
As before we see that the transductive probably submodular constraint set \Ctrans{} does not lead to significant improvements over $\Cf$, validating our approach and leading to the conclusion that the geometry of the dataset necessary for efficient inference is accurately captured by the training set alone. 
Table~\ref{tradeoff} indicate that only a small amount of edges are non-submodular at test time for the two probably submodular constraint sets, indicating that the truncated graph cut inference is very close to being exact.

We also compare with using an approximate solver for the SSVM learning, without submodularity constraints: we use QPBO, which corresponds to the extended roof duality solver of \citep{rother2007optimizing}. 
We see that the QPBO solver converges to a suboptimal solution, with a lower test accuracy than the SSVM optimized with definitely submodular ($\Co$) or probably submodular ($\Cf$) constraints.

Figure~\ref{fig:FelineSegmentations} shows an example of a segmented retinal image and highlight the gains obtained by training the model with more expressive constraint sets.

\subsection{Mitochondria
segmentation from 3D electron microscopy}\label{sec:ApplicationMitochondria}
\begin{figure*}[ht]
\centering
\subfigure[Ground truth]{
      \includegraphics[width=0.23\textwidth]{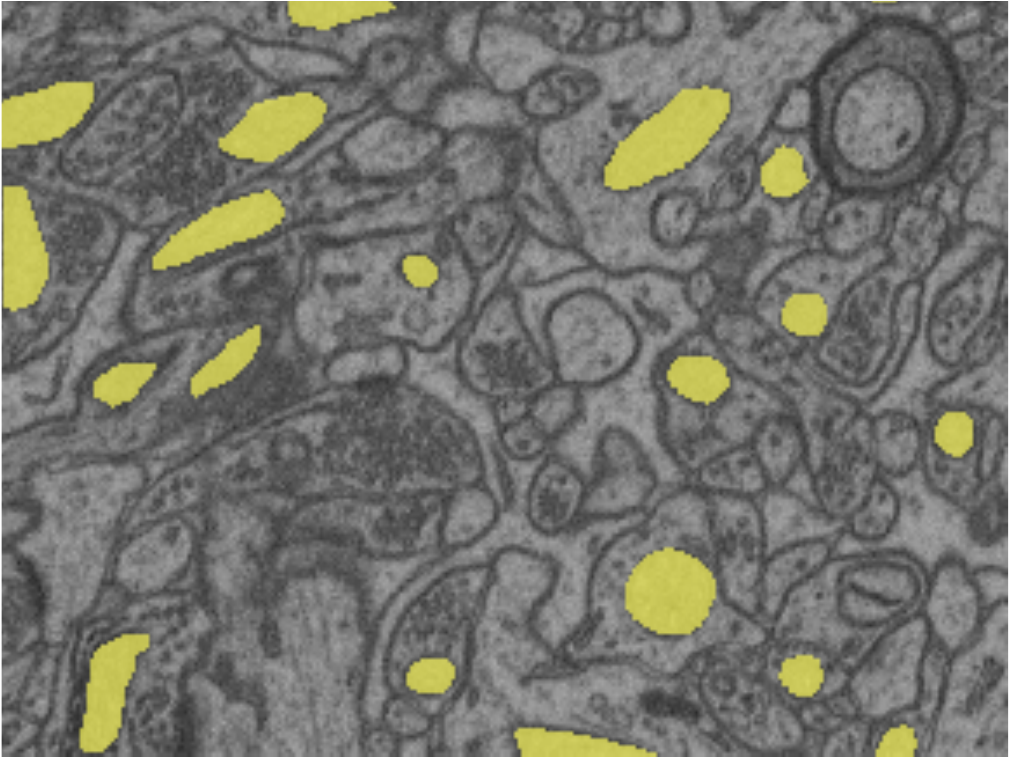}
}
\subfigure[$\C_0$]{
      \includegraphics[width=0.23\textwidth]{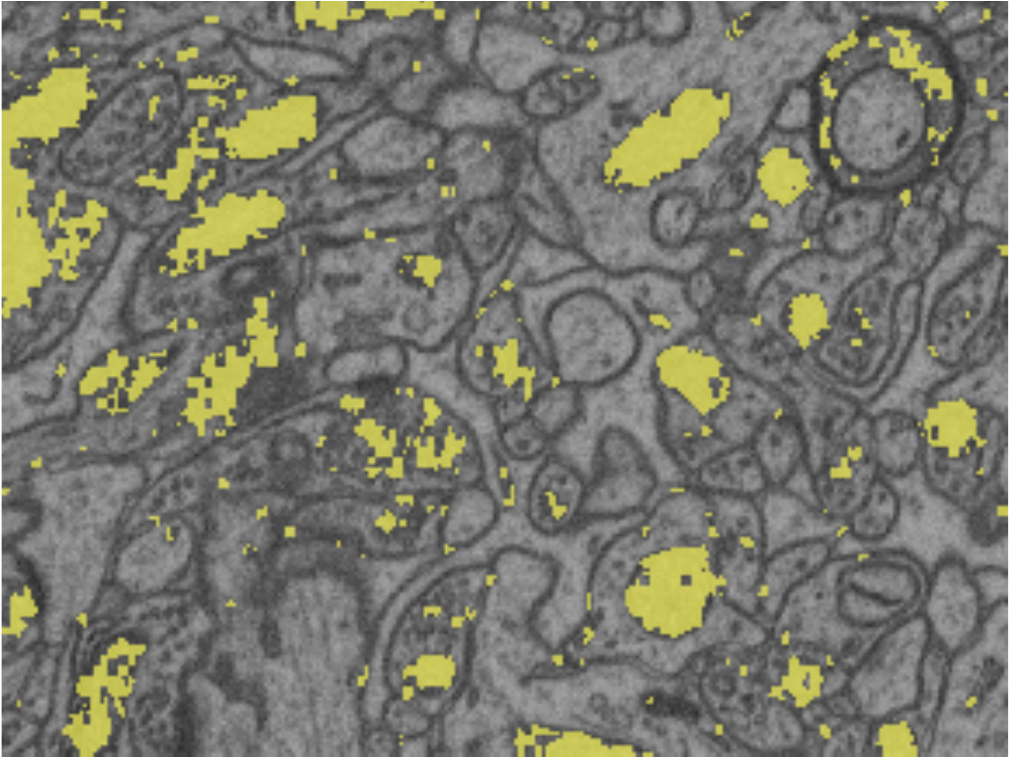}
}
\subfigure[$\C_1$]{
      \includegraphics[width=0.23\textwidth]{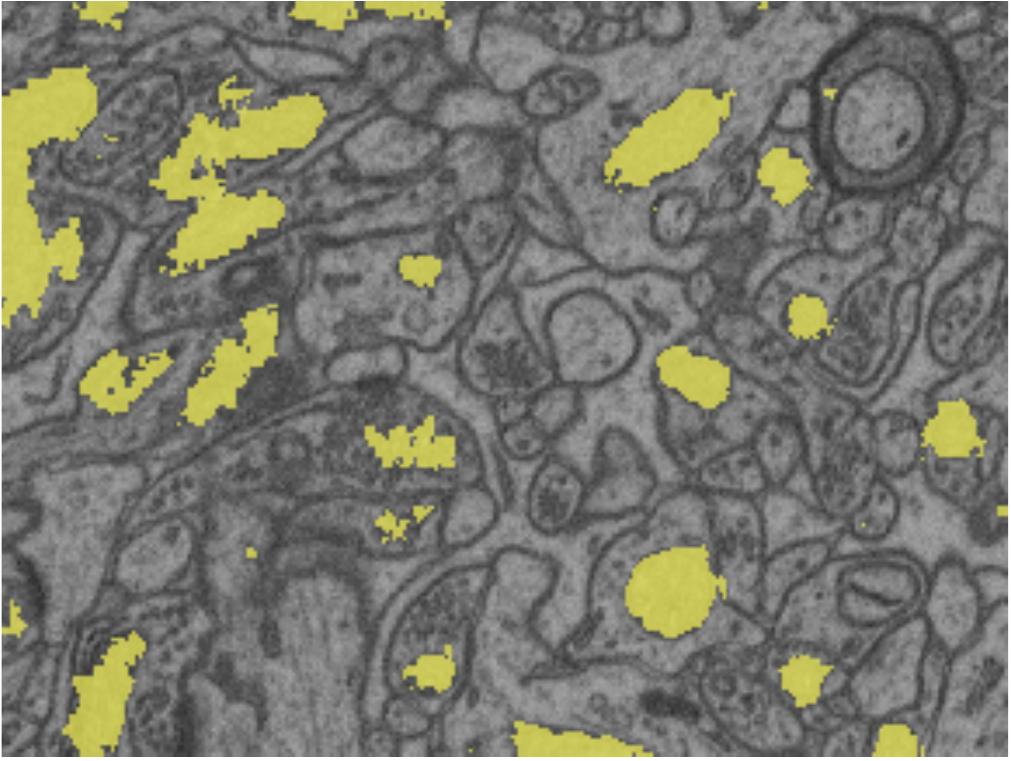}
}
\subfigure[$\Co$]{
      \includegraphics[width=0.23\textwidth]{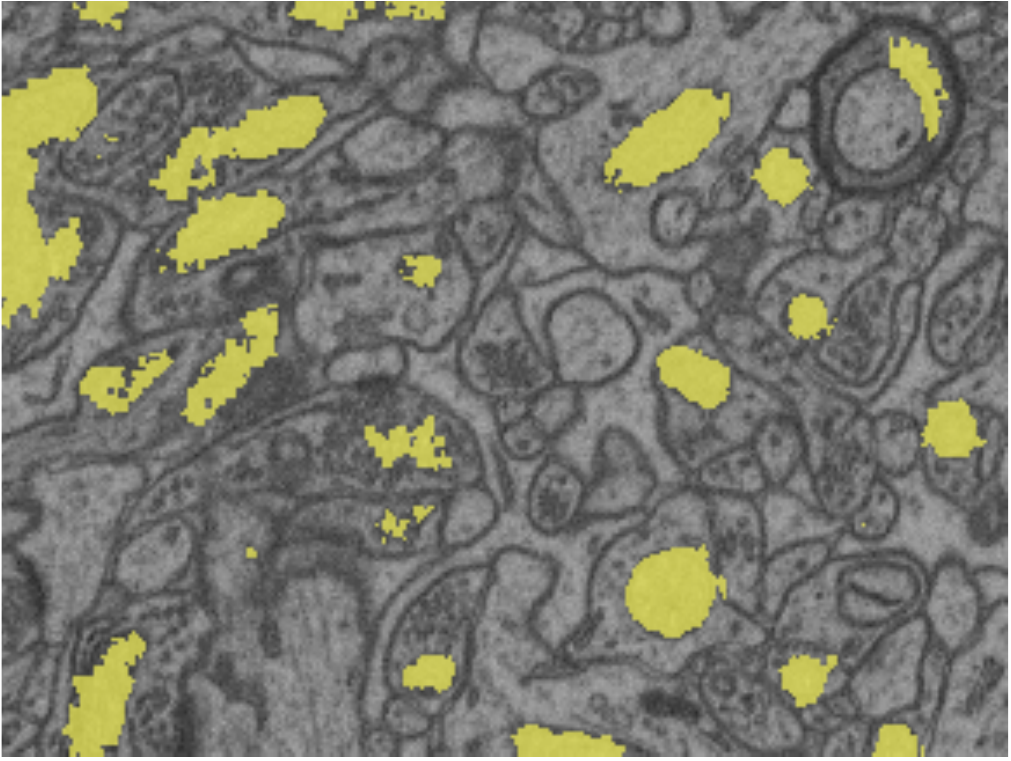}
}
\subfigure[$\C_3$]{
      \includegraphics[width=0.23\textwidth]{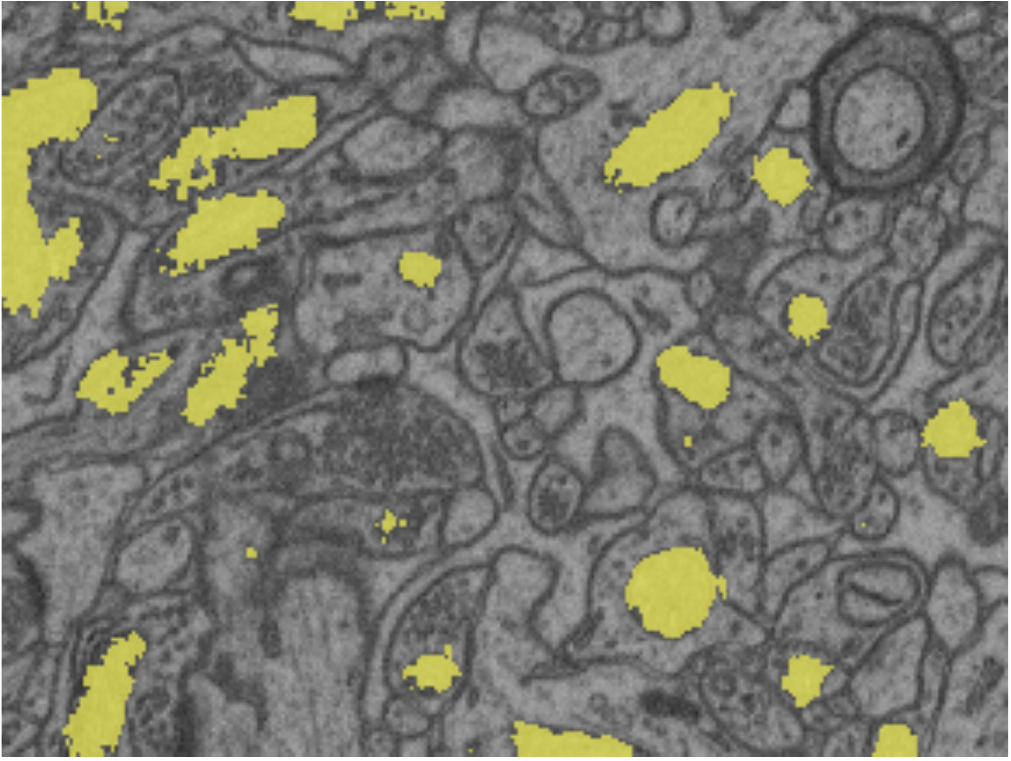}
}
\subfigure[$\Cf$]{
      \includegraphics[width=0.23\textwidth]{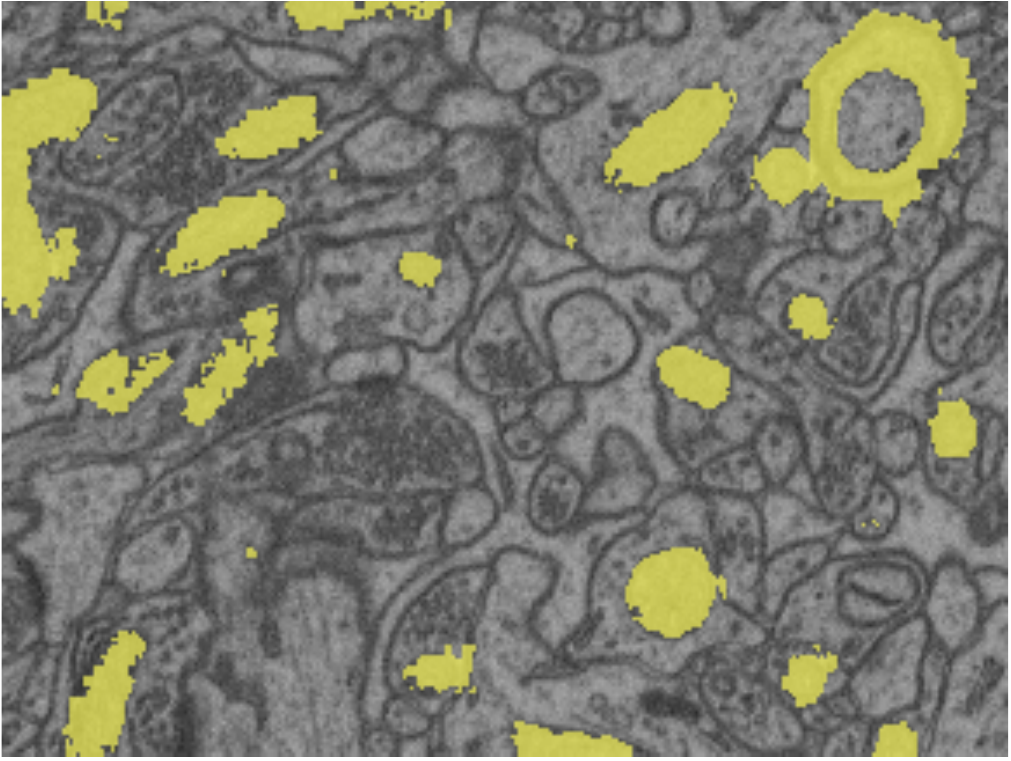}
}
\subfigure[QPBO]{
      \includegraphics[width=0.23\textwidth]{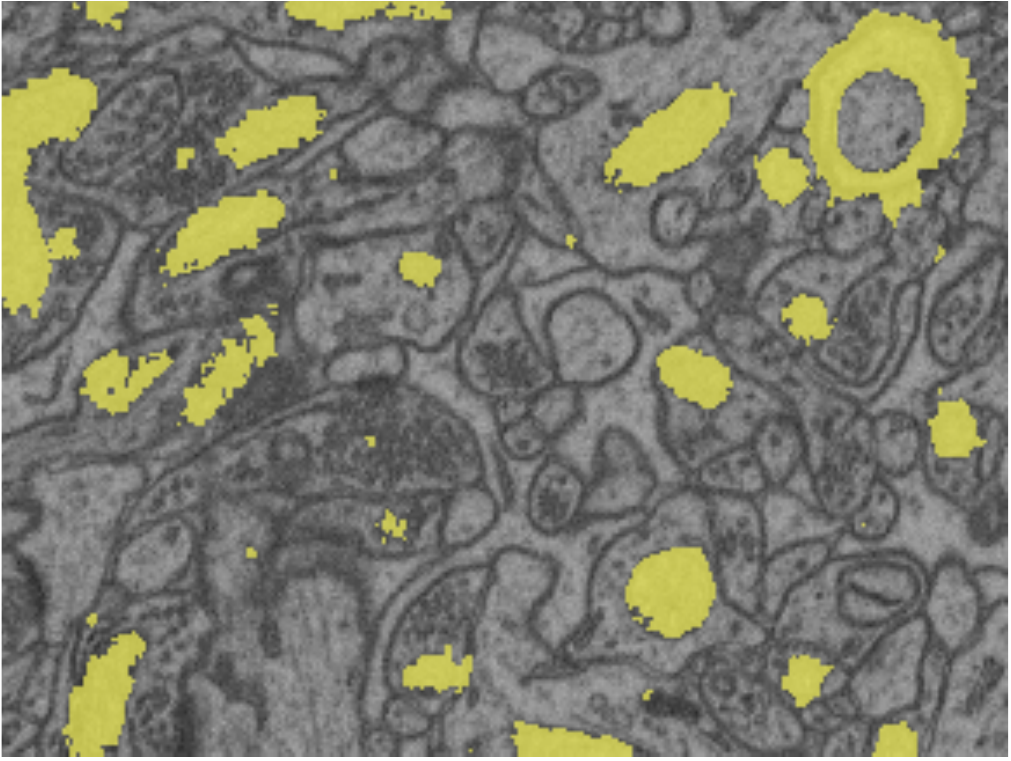}
}
\caption{Slice of the test segmentation volume obtained for the Mitochondria volumetric segmentation task.
.}\label{fig:MitoSegmentations}
\end{figure*}
\begin{table*}[ht]
\centering
\caption{Train and test losses and test performances of the segmentations of the 3D electron microscopy images of mitochondria in the hippocampus.
 ``Oracle'' corresponds to the supervoxel oracle performance. 
 ``SSVM iter.'' reports the number of SSVM iterations of the method.}
\label{table:mito}
\begin{tabular}{@{}lccccccccc@{}}
\toprule
 & $\C_0$ & $\C_1$ & $\Co$ & $\C_3$ & $\Cf$ & \Ctrans{} & TGC & QPBO & Oracle \\
\midrule
$1 - \Delta_{\text{avg}}^{\text{train}} (\%)$ & 87.6  & 96.9 & 97.0 & 97.4 & \textbf{98.0} & \textbf{98.0} & 85.2 & 97.9  &  \\
$1 - \Delta_{\text{avg}}^{\text{test}} (\%)$ &  87.7  & 95.9  &  95.6 & \textbf{96.5}  & 96.2 & 96.1  & 78.8 & 96.0 & 100\\
IoU (\%) &  43.9 & 60.4  & 58.0 & \textbf{64.5} & 56.9 & 56.6 & 51.9 & 58.0 & 100 \\ 
VOC (\%) &  69.1  &  78.5  & 77.1 & \textbf{80.8} & 76.5 & 76.3 & 74.5 & 77.1 &  90.05\\
\midrule
SSVM iter. & 369 & 1082 & 1303 & 976 & 3981 & 3322 & 25 & 5000 \\
\bottomrule
\end{tabular}
\end{table*}

We perform  for structured learning in a publicly available 3D electron microscopy image taken from the CA1 hippocampus region of the brain~\citep{lucchi2012supervoxel}, with annotation of the mitochondria. This dataset has been shown to benefit from structured learning over unary voxel-based non-structured approaches~\citep{lucchi2015learning}, and consists of two greyscale volumes containing $1024 \times 768 \times 165$ voxels, one being used for training and one being used for testing.
We perform initial supervoxel segmentation with $\sim 130\text{K}$ supervoxels with compacity $m = 20$, after initial gaussian smoothing of the image volumes with $\sigma = 5$. To cross-validate the regularization parameter of the different methods, we split the training volume into two $512 \times 768 \times 165$ sub-volumes and perform validation of the $C$ among $11$ values logarithmically distributed between $10^{-2}$ and $10^3$ for each method considered. 
For each superpixel, we extract unary features as the concatenation of $40$-bin intensity histograms, a $8\times 8$ grey-level co-occurence matrix with $1$ pixel displacement in the three dimensions, and a bias channel of constant value $1$. 
We use the absolute difference between unary features as pairwise features. 
The SSVM is optimized with a per-class average loss, with a final bound of $10^{-3}$ on the primal-dual cutting-plane relative optimization gap and a maximum number of 5000 iterations.

We report the accuracies of the SSVM prediction on the train and test set after optimization with the different constraint sets considered, in terms of the test set superpixel loss $\Delta_{\text{avg}}$ 
in Table~\ref{table:mito}. 
In addition, we compare with the QPBO approximate solver (as in Section \ref{sec:ApplicationPhotonReceptorCells2Dretinal}), as well as truncated Graph-Cuts with no submodularity constraints (TGC). 
We also report the pixelwise IoU and VOC score on the test set, as well as the supervoxel oracle performance $\mathcal{O}$, which is the best performance one may obtain using these supervoxels.

As before, we see an improvement associated to the use of larger optimization domains for the model parameter. 
Only a small fraction $3.5\cdot 10^{-05}$ of non-sub\-mod\-ular edges are present on the test set after optimization under the probably submodular constraint set $\Cf$.
While the training accuracy (objective) is increasing when relaxing constraints from $\C_0$ to $\Cf$ as expected, we notice that in this case the test accuracy on $\C_3$ is greater than the test accuracy on $\Cf$, especially in terms of IoU and VOC score. 
In this case, we believe that the added constraints in $\C_3$ can have a regularizing effect that helps the generalization of the model. 

While using a truncated graph-cut solver leads to subpar performance on this problem, we see that the QPBO solver reaches competitive accuracies. 
However, it is also the only method that does not converge within 5000 SSVM iterations. 
Inexact algorithms can lead to a poor estimation of the primal-dual optimization gap, on which our termination criterion is founded. 

A visualization of a slice of the segmentation volumes resulting from the different methods is given in Figure~\ref{fig:MitoSegmentations}.

\subsection{Multi-label classification datasets}\label{sec:ExperimentsMultiLabelClassification}
\begin{table*}
\centering
\caption{Train and test accuracies for the multi-label classification datasets, with standard error bars. \label{table:multilabel}}
\begin{tabular}{@{}clccccccc@{}}
\toprule
dataset & & $\C_0$ & $\Co$ & $\Cf$ & TGC & QPBO & TRW-S & (Fin--Best) \\
\midrule
\multirow{2}{*}{\texttt{yeast}} & $1 - \Delta^{\text{train}} (\%)$ & 69.8 $\scriptstyle \pm {.5}$ & \textbf{80.7} $\scriptstyle \pm {.4}$ & \textbf{80.7} $\scriptstyle \pm {.4}$ & 72.6 $\scriptstyle \pm {.6}$ & 67.7 $\scriptstyle \pm {.6}$ & 70.5 $\scriptstyle \pm {.6}$ & \\
 & $1 - \Delta^{\text{test}} (\%)$ & 69.6 $\scriptstyle \pm {.6}$ & \textbf{80.0} $\scriptstyle \pm {.4}$ & \textbf{80.0} $\scriptstyle \pm {.4}$ & 71.8 $\scriptstyle \pm {.6}$ & 66.9 $\scriptstyle \pm {.6}$ & 69.6 $\scriptstyle \pm {.6}$ & 79.8 $\scriptstyle \pm {.5}$  \\
 
\multirow{2}{*}{\texttt{scene}} & $1 - \Delta^{\text{train}} (\%)$ & 85.5 $\scriptstyle \pm {.4}$  & 93.7 $\scriptstyle \pm {.3}$  & \textbf{94.3} $\scriptstyle \pm {.2}$  & 83.8 $\scriptstyle \pm {.5}$  & 74.1 $\scriptstyle \pm {.5}$  & 74.2 $\scriptstyle \pm {.3}$ \\
 & $1 - \Delta^{\text{test}} (\%)$ & 84.5 $\scriptstyle \pm {.4}$ & 90.2 $\scriptstyle \pm {.3}$ & \textbf{90.4} $\scriptstyle \pm {.3}$& 83.3 $\scriptstyle \pm {.4}$ & 72.8 $\scriptstyle \pm {.5}$ & 72.8 $\scriptstyle \pm {.5}$ & 89.9 $\scriptstyle \pm {.3}$\\

\bottomrule
\end{tabular}
\end{table*}

We evaluate our structured prediction learning framework in multi-label classification settings, as described in Section~\ref{sec:MultiLabelClassificationApplicationSetting}. We use the $\texttt{yeast}$~\citep{Elisseeff2001AKM} dataset, which has $1500$ training and $917$ test instances with $d=103$ attributes and $|\mathcal{C}|=103$ classes, and the $\texttt{scene}$~\citep{Gjorgjevikj2011TwoSC} dataset, which has $1211$ train and $1196$ test instances with $d=294$ attributes and $|\mathcal{C}|=6$ classes. 
For each input $\vec{x}$, we construct the edge features $\vec{R}(\vec{x})$ (Equation \eqref{eqmultilabelR}) by a dimensionality reduction of $\vec{x}$. 
Specifically, we transform $\vec{x}$ into an 20-dimensional PCA reduction $\vec{\tilde{x}}$ learned on the training set, and set $\vec{R}(\vec{x})$ to be the concatenation of the positive and negative part of $\vec{\tilde{x}}$:
\begin{equation}
    $\vec{R}(\vec{x})$=\max(\vec{\tilde{x}}, 0) \oplus \max(-\vec{\tilde{x}}, 0),
\end{equation}
satisfying the feature positivity assumption of our constrained SSVM framework.

Table~\ref{table:multilabel} compares the resulting accuracies after training with the Hamming loss, using a SSVM regularization parameter $C=0.1$ and a final bound on the relative optimization gap of $10^{-2}$ and a maximum number of $200$ iterations. 
$\mathcal{C}_0$ is an SVM trained with a unary-only solver. 
$\mathcal{C}_2$ and $\mathcal{C}_4$ corresponds to our submodular and probably submodular constraint set, leading to an exact graph-cut inference.  
We also report the accuracies of models learned without submodular constraints using approximate inference algorithms, namely truncated graph-cuts (TGC), QPBO and the TRW-S tree-reweighted algorithm of \cite{Kolmogorov2005ConvergentTM}. 
All methods converge in less than $200$ SSVM optimization iterations, except the TRW-S-based model on the \texttt{yeast} dataset.

Approximate algorithms leads to bad solutions on these multi-label classification problems -- sometimes worse than the accuracy of $\C_0$ which uses unaries alone. 
This exemplifies the need for exact inference algorithms in structured prediction learning problems, and the merit of extending the relaxing the constraints of exact models through the probably submodular framework. 

Finally, we report accuracies obtained by \citep{Finley2008TrainingSS}, which also uses a SSVM learning setting, only with simple indicator edge features corresponding to $\vec{R}(\vec{x}) = 1$, (where the probably submodular and the definitely submodular constraints sets are the same).
We report their best reported accuracies 
among various inference algorithms (Fin--Best).
While the error bars of our method ($\mathcal{C}_4$) overlap w.r.t.\ \citet{Finley2008TrainingSS}, we note that Fin-Best uses a brute-force exact algorithm, which cannot scale to larger problems, while our algorithm can, thanks to fast graph-cut inference.

\section{Conclusions}
We present the probably submodular framework, which allow to learn more expressive CRF models without sacrificing tractability, through the use of exact graph-cut based inference routines.
This methods leads to more expressive models without paying the price of tractability.
We presented efficient optimization strategies for the probably submodular structured SVM framework.
Although we have presented our approach in the case of exact optimization routines of the learning objective, through the duality certificates of the SSVM cutting-plane optimization, our method can be extended to first-order optimization schemes such as projected sto\-chas\-tic gradient descent.

\renewcommand{\UrlBreaks}{\do\/}
\mathchardef\UrlBreakPenalty=1000\relax

Our algorithms were implemented as additions to the Python module for structured prediction \emph{PyStruct} of~\citet{software:pystruct}. The code and experiments have been made available on the repository \sourcecodeurl .

\FloatBarrier 

\begin{acknowledgements}
This work is partially funded by
Internal Funds KU Leuven and an Amazon Research Award. This work made use of a hardware donation from the Facebook GPU Partnership program, and an NVIDIA GPU donation.
We
acknowledge support from the Research Foundation - Flanders
(FWO) through project number G0A2716N.  This  research  received  funding  from  the
Flemish Government under the Onderzoeksprogramma Artifici\"{e}le Intelligentie
(AI) Vlaanderen programme.
The authors thank Wojciech Zaremba for making this work possible.
\end{acknowledgements}

\bibliographystyle{spbasic}      

\bibliography{bermanbib,blaschkobib,bibliography}

\end{document}